\pdfoutput=1
\documentclass{article}
\usepackage[T1]{fontenc}
\usepackage{times}

\usepackage[left=31mm,right=31mm,top=33mm,bottom=20mm]{geometry}
\usepackage{natbib}
\usepackage{booktabs}       

\usepackage{xcolor}         
\usepackage[colorlinks,
linkcolor=blue,
citecolor=blue,
urlcolor=magenta,
linktocpage,
plainpages=false]{hyperref}

\bibpunct{(}{)}{;}{a}{,}{,}

\usepackage{amsmath,amssymb,amsfonts}
\usepackage{amsthm}
\usepackage{thmtools,thm-restate} 
\usepackage{enumitem}
\usepackage{bbm}
\usepackage{dsfont}
\usepackage{graphicx}
\usepackage{epsfig}
\usepackage{subcaption}
\usepackage{amsfonts}
\newtheorem{theorem}{Theorem}
\newtheorem{lemma}[theorem]{Lemma}

\theoremstyle{definition}
\newtheorem{definition}[theorem]{Definition}

\newtheorem{observation}[theorem]{Observation}

\theoremstyle{remark}

\usepackage[noend,ruled]{algorithm2e}
\SetCommentSty{mycommfont}
\SetKwInput{KwInput}{Input}                
\SetKwInput{KwOutput}{Output}
\SetKwInput{KwInit}{Initialization}            

\SetKwComment{Comment}{/* }{ */}

\DeclareMathOperator*{\argmax}{arg\,max}

\DeclareMathOperator*{\nth}{^{\text{th}}}

\newcommand{\E}{\mathbb{E}}

\newcommand{\indicator}{\mathds{1}}

\newcommand{\A}{\mathcal{A}}

\newcommand{\regret}{\mathcal{R}}

\newcommand{\F}{\mathcal{F}}

\title{On the price of exact truthfulness in incentive-compatible online learning with bandit feedback: \\A regret lower bound for WSU-UX}

\author{Ali Mortazavi \\ 
             University of Victoria \\
             \texttt{alithemorty@gmail.com}
             \and
             Junhao Lin\thanks{Work completed while at University of Victoria} \\ 
             University of Waterloo \\
             \texttt{linjunhao9385@gmail.com}
             \and
             Nishant A. Mehta \\ 
             University of Victoria \\
             \texttt{nmehta@uvic.ca}}

\date{}

\begin{document}

\maketitle

\begin{abstract}
  In one view of the classical game of prediction with expert advice with binary outcomes, in each round, each expert maintains an adversarially chosen belief and honestly reports this belief. We consider a recently introduced, strategic variant of this problem with selfish (reputation-seeking) experts, where each expert strategically reports in order to maximize their expected future reputation based on their belief. In this work, our goal is to design an algorithm for the selfish experts problem that is incentive-compatible (IC, or \emph{truthful}), meaning each expert's best strategy is to report truthfully, while also ensuring the algorithm enjoys sublinear regret with respect to the expert with the best belief. \citet{freeman2020no} recently studied this problem in the full information and bandit settings and obtained truthful, no-regret algorithms by leveraging prior work on wagering mechanisms. While their results under full information match the minimax rate for the classical ("honest experts") problem, the best-known regret for their bandit algorithm WSU-UX is $O(T^{2/3})$, which does not match the minimax rate for the classical ("honest bandits") setting. It was unclear whether the higher regret was an artifact of their analysis or a limitation of WSU-UX. We show, via explicit construction of loss sequences, that the algorithm suffers a worst-case $\Omega(T^{2/3})$ lower bound. Left open is the possibility that a different IC algorithm obtains $O(\sqrt{T})$ regret. Yet, WSU-UX was a natural choice for such an algorithm owing to the limited design room for IC algorithms in this setting.
\end{abstract}

\section{Introduction}
In the problem of prediction with expert advice \citep{vovk1995game}, we have $K$ experts and $T$ days and a fixed loss function $\ell(x_t, y_t)$. Each day $t \in [T]$, the learner has access to the advice $b_{i,t}$ of each expert $i \in [K]$ about the outcome $y_t$. Using the expert advice, the learner predicts $x_t$ and then outcome $y_t$ is revealed. The error of the learner is measured by $\ell(x_t, y_t)$ at round $t$. The goal of the learner is to achieve low regret with respect to the cumulative loss of the best expert in hindsight. One common approach to this problem in the literature is to, for any round $t$, maintain a set of weights over experts $w_{t,i}$ for all $i \in [K]$, select the advice of expert $i$ with probability $\frac{w_{t,i}}{\lVert \boldsymbol{w}_t \rVert}_1$, and update the weights appropriately once the outcome $y_t$ is revealed. See for example the multiplicative weight update (MWU) method \citep{arora2012multiplicative} and Hedge \citep{freund1997decision}. 

\cite{roughgarden2017online} considered this problem for binary outcomes (i.e.~$y_t \in \{0,1\}$) where the experts are strategic: in each round $t$, each expert forms a belief $b_{i,t} \in [0,1]$ about the binary outcome $y_t$, (i.e.~$b_{i,t}=\Pr\left(y_t=1\right) \in [0,1]$) and reports $r_{i,t} \in [0,1]$ in such a way as to maximize its own future \textit{reputation} among the pool of experts. See the protocol in Algorithm~\ref{alg:reputation-seeking-full-info-protocol} which shows this framework.

\begin{algorithm}[t]
  \DontPrintSemicolon
  \KwInput{$T$, $K$, $\ell(X,Y): [0,1] \times \{0,1\} \rightarrow \mathbb{R}$}
  \For{$t = 1, \ldots, T$}
  {
    The learner chooses a distribution $\boldsymbol{\pi_{t}} \in \Delta_{K}$ over experts and draws an expert $I_t$.
    
    Each expert $i \in [K]$ forms a belief $b_{i,t} \in [0,1]$ about the distribution of outcome $y_t$.
    
    Each expert $i$ reports a prediction $r_{i,t} \in [0,1]$ with the goal of maximizing~their~own~\textit{future~reputation}.

    Nature reveals the outcome $y_t \in \{0,1\}$.
    
    Learner incurs loss of $\E[\ell(r_{I_t,t},y_t)] =\sum_{j\in [K]} \pi_{t,i}\ell(r_{j,t},y_t)$
    
  }
  
  \caption{Protocol for Prediction With Selfish (Reputation Seeking) Experts for binary outcomes}
  \label{alg:reputation-seeking-full-info-protocol}
\end{algorithm}

Moreover, considering the class of learning algorithms that maintain weights over experts $w_{t,i}$ for $i \in [K]$, \cite{roughgarden2017online} assumed that each expert $i$ at round $t$ associates {its own current reputation with the weight $w_{t,i}$ and} its future reputation simply as its weight in the next round $w_{t+1,i}$. This type of expert is called a \emph{myopic} expert as it does not consider the impact of the decision on its long-term reputation.  

Given this notion of future reputation, the design of the algorithm would impact how each expert would report. Consider round $t$ and expert $i$ and a fixed weight-based learning algorithm; depending on the learning algorithm, there is a function $f$ that determines the weight $w_{t+1,i}$ as
\begin{align}
  w_{t+1,i}= f(r_{i,t}, y_t,  r_{-i, t}, h_{t-1}),
  \label{eq:some-eq44}
\end{align}
where $r_{-i, t}$ denotes the reports of the experts other than expert $i$ and $h_{t-1}$ is all the information revealed by the end of round $t-1$. 

Moreover, assume that expert $i$ has perfect information about $r_{-i, t}$ and $h_{t-1}$. 
Assuming expert $i$ has a belief $b_{i,t}$ about the distribution of $y_t$, then it reports $r_{i,t}$ to maximize its expected reputation:
\begin{align*}
  r_{i,t} &= \argmax_{r \in [0,1]}\E_{b_{i,t}}\left[w_{t+1,i}\right] \\
          &=  \argmax_{r \in [0,1]}\E_{b_{i,t}}\left[f(r, y_t, r_{-i, t}, h_{t-1})\right],
\end{align*}
where the expectation is over the randomness of the outcome $y_t$ as if it were drawn based on the expert's belief $b_{i,t}$. 
Observe that depending on function $f$, the value the expert $i$ reports $r_{i,t}$ can be different than its belief $b_{i,t}$. If truth-telling is always a dominating strategy, meaning that no matter the other experts' reports $r_{-i,t}$, the best response is $r_{i,t}=b_{i,t}$, then the algorithm is called \emph{incentive-compatible}. The design of incentive-compatible algorithms is desirable for two reasons:
\begin{itemize}[topsep=0pt,itemsep=0pt,itemindent=0pt,leftmargin=1.5em]
\item Quality of prediction: The regret guarantee for an incentive-compatible online learning algorithm holds not only for the expert with the best reports but also holds for the expert with the best beliefs, which the algorithm does not have direct access to. This guarantee is called
  \textit{belief regret}.\footnote{\cite{frongillo2021efficient} used the term ``regret with respect to the true beliefs''.}
\item Natural strategy: An expert does not need to take into consideration the reports of other experts. Moreover, when a simple strategy (truth-telling) is strictly dominating, it is reasonable to expect that agents will choose that strategy.
\end{itemize}

As observed by \cite{roughgarden2017online}, the design of incentive-compatible online learning algorithms is intimately connected to the problem of designing proper scoring rules (see Definition~\ref{def:proper_scoring_rule} in Section~\ref{sec:preliminaries}). This implies that when the loss function is proper, the problem is easy. For instance, for any proper loss function, such as squared loss, MWU which uses the update rule
\begin{align*}
  w_{t+1,i} = w_{t,i} (1-\eta \ell(r_{i,t}, y_t)),
\end{align*}
is incentive-compatible.

However, for absolute loss, which is not a proper loss function, \citet[Corollary 31]{roughgarden2017online} showed that, under some mild restrictions\footnote{The class of algorithms that are considered by \cite{roughgarden2017online} are those that are a natural extension of deterministic weighted majority algorithms, where the weight update has some mild restrictions.}, no weight-based randomized algorithm can achieve no-regret.

Yet, as observed by \cite{freeman2020no}, even for proper loss functions such as squared loss, for another natural variation of incentive for experts who, in any round $t$, want to maximize their expected \textit{normalized} weight (i.e.~the probability of being selected) in the next round based on their private belief about the outcome $y_t$, the classical multiplicative weight algorithm (MWU) fails to be truthful\footnote{It is easy to show that Hedge and MWU are not truthful. However, note that \cite{frongillo2021efficient} showed that Hedge is approximately truthful. Under some assumptions, approximately truthfulness is enough to get good belief regret, but in this paper, we only focus on exactly-truthful/incentive-compatible algorithms: the algorithms where truth-telling is the only dominant strategy.}. \cite{freeman2020no} designed the Weighted-Score Update rule which is truthful and also achieves $O(\sqrt{T})$ regret in the full-information setting. However, in their extension to the multi-armed bandit setting (which they refer to as the partial information setting), their algorithm WSU-UX achieves $O(T^{2/3})$ regret, which does not match the minimax optimal rate $O(\sqrt{T})$ in the classical ``honest experts'' problem.

Although the experimental results by \cite{freeman2020no} suggested that WSU-UX performs similarly to EXP3 \citep{auer2002nonstochastic} which has minimax optimal regret of $O(\sqrt{T})$ in the classical ``honest experts'' problem, it remained an open problem whether the $O(T^{2/3})$ regret of WSU-UX is due to an artifact in the analysis or if instead the algorithm cannot achieve lower regret. 

\paragraph{Main question}
The main question that we are interested in understanding is: in the setting of \cite{freeman2020no}, whether learning with reputation-seeking experts under bandit feedback is strictly harder than the classical bandit problem.

\paragraph{Contribution}
We take one step toward answering this question by showing that WSU-UX, which is a very natural choice for this problem, can not achieve regret better than $\Omega(T^{2/3})$ in the worst case. In particular, we show that for any choice of hyperparameters for WSU-UX, for large enough $T$, there exists a loss sequence where the belief regret is $\Omega(T^{2/3})$.
The construction of the loss sequence is fairly simple but, for the set of non-trivial hyperparameters (see Section~\ref{sec:lower_bound_high_level} for a description of this set) requires a highly intricate analysis. 
In particular, for the non-trivial case, we design a loss sequence such that (i) the best expert\footnote{The best expert here is the expert whose belief has the lowest cumulative loss over $T$ rounds.} has an estimated loss with large variance for a constant fraction of $T$ rounds, and at the same time (ii) the best expert outperforms the other experts at the end. The core technical difficulty is showing that both (i) and (ii) happen simultaneously.

\section{Model and preliminaries}

\subsection{Problem setting}
We first describe the learning protocol. In the setting of \cite{freeman2020no}, the protocol is the same as 
Protocol~\ref{alg:reputation-seeking-full-info-protocol}, where the loss function is squared loss (which is a proper loss). In the full-information version of the problem, all experts offer reports, while under bandit feedback, only the selected expert offers a report. The future reputation of each expert is defined as the probability of being selected in the next round. More concretely, each expert $i$ at round $t$ with belief $b_{i,t} \in [0,1]$ about binary outcome $y_t$ strategically reports $r_{i,t} \in [0,1]$ to maximize their probability of being selected by the algorithm in the next round $t+1$. The goal of the learner is to minimize its belief regret, the regret with respect to the expert with the best belief.
\begin{definition}
  Let $\boldsymbol{\pi}_t \in [K]$ be the learner's probability distribution in round $t$.   Then the learner's \emph{belief regret} $\E[\regret_T]$  after $T$ rounds is defined as
  \begin{align*}
    \E\left[\sum_{t\in [T]}\sum_{j \in [K]}\pi_{t,j} \ell(r_{j,t},y_t) - \min_{i \in [K]}\sum_{t \in [T]}\ell(b_{i,t},y_t)\right].
  \end{align*}
\end{definition}
Note that the learner incurs loss according to reports $r_{i,t}$ whereas the performance of the best expert is measured with respect to $b_{i,t}$. In general, $r_{i,t}$ need not equal $b_{i,t}$. However, if a learning algorithm is incentive-compatible, meaning that truth-telling is the only strictly dominant strategy, it is reasonable to assume that $r_{i,t}=b_{i,t}$. In this case, low classical regret implies low belief regret. Next, we restate the definition of incentive-compatibility from \cite{freeman2020no}. 
\begin{definition}[\cite{freeman2020no}]
  An online learning algorithm is incentive-compatible if for every timestep $t\in [T]$, every expert $i$ with belief $b_{i,t}$, every report $r_{i,t}$, every vector of reports of the other experts $r_{-i,t}$, and every history of reports $(r_{t'})_{t'<t}$ and outcomes $(y_{t'})_{t'<t}$,
  \begin{align*}
    &\E_{y_t \sim \mathrm{Bern}(b_{i,t})}\left[\pi_{t+1,i} \mid (b_{i,t}, r_{-i,t}), y_t, (y_{t'})_{t' < t}, (r_{t'})_{t'<t}\right] \\
    & \geq 
      \E_{y_t \sim \mathrm{Bern}(b_{i,t})}\left[\pi_{t+1,i} \mid (r_{i,t}, r_{-i,t}), y_t, (y_{t'})_{t' < t}, (r_{t'})_{t'<t}\right], 
  \end{align*}
  where $y \sim \mathrm{Bern(b)}$ denotes a random variable taking value 1 with probability $b$ and 0 otherwise. 
\end{definition}
\subsection{Motivation for bandit setting}
To motivate the bandit version of the problem, consider the following example.
A forecasting agency wishes to forecast an event and has a choice of which
forecaster to employ. 
The selected forecaster will be given a fixed payment (say \$1000) from the agency to research the event, will then develop their belief about the likelihood of the event occurring, and will finally decide what probability forecast (report) to give the forecasting agency. 
Any forecaster that is not selected will not receive a payment and will never provide a report to the forecasting agency. 
Naturally, the agency desires accurate reports, and so its goal is to select the forecaster whose belief (which can be developed only after the forecaster is funded and hence was selected) is the most accurate.
To this end, the agency should ensure that the \emph{future} expected payment given to a selected forecaster (for the next event) incentivizes the forecaster to report its belief honestly for the current event. 
The forecaster's incentive is exactly equal to the future probability of being selected
(a quantification of the forecaster's reputation) since, if selected in the future, the selected forecaster will again receive a fixed payment.
The agency should thus ensure that a forecaster's future probability of being selected is directly proportional to the accuracy of the forecaster's report, an accuracy which is known once the outcome has been realized.

\subsection{Preliminaries} \label{sec:preliminaries}

In this subsection, we overview fundamental concepts relevant to incentive-compatibility. We first recall the notion of proper scoring rules, which can be used to elicit information from an expert \citep{gneiting2007strictly}, \citep{buja2005loss}.

\begin{definition}
  Let $\mathcal{Y}$ denote the outcome space and $\mathcal{R} \subseteq  \Delta({\mathcal{Y}})$ denote the distributional report space. A scoring rule $s: \mathcal{R} \times \mathcal{Y} \rightarrow \mathbb{R}$ is \textit{proper} if for any $b,r \in \mathcal{R}$, we have 
  \begin{align*}
    \E_{Y \sim b}\left[s(b, Y)\right] \geq \E_{Y \sim b}\left[s(r, Y)\right],
  \end{align*}
  and \textit{strictly proper} if the inequality becomes tight only when $r=b$. 
  \label{def:proper_scoring_rule}
\end{definition}
  
  This implies that when a scoring rule is (strictly) proper, an expert with belief $b$ about the distribution of outcome $Y \in \mathcal{Y}$ would (uniquely) maximize their expected score by reporting $r=b$.
  
  A (strictly) proper loss function $\ell$ is defined similarly where truthful reporting of the belief (strictly) \textit{minimizes} the expected loss. We assume $\mathcal{R} = [0,1]$ and $\mathcal{Y} = \{0,1\}$ in this paper.

Next, let us view any online algorithm for the problem of prediction with expert advice as follows.

\begin{definition}[Probability-based update class of an online learning algorithm]
  An online learner $M$ maintains a distribution $\boldsymbol{\pi}_{t}$ over $K$ experts. In round $t$, the learner draws an expert $I_t=i$ with probability $\pi_{t,i}$. Then, the outcome $y_t$ is revealed, and the learner incurs loss $\ell_{I_t,t} = \ell(r_{I_t,t}, y_t)$ and updates $\boldsymbol{\pi}_{t+1}$ only as a function of $\boldsymbol{r}_t = (r_{t,1}, \ldots, r_{t,K})$, $\boldsymbol{\pi}_{t} = (\pi_{t,1}, \ldots, \pi_{t,K})$, and outcome $y_t$.
  \label{def:prob-based-description}
\end{definition}
Note that many algorithms can be described as probability-based update algorithms, such as MWU and Hedge. In order to make a distinction, we describe the precise definition of Hedge and MWU using the probability-based update description. 
\begin{definition}
  Hedge initializes the weights $w_{1,i}=\frac{1}{K}$ for all $i\in [K]$, updates the weights in each round based on the update
  \begin{align}
    w_{t+1,i} = w_{t,i} \cdot \exp{\left(-\eta \ell(r_{i,t}, y_t)\right)},
    \label{eq:some-eq50}
  \end{align}
  and chooses $\pi_{t,i} = \frac{w_{t,i}}{\sum_{j}w_{t,j}}$. 
  
  MWU does the same except it uses the update
  \begin{align*}
    w_{t+1,i} = w_{t,i} \cdot (1-\eta \ell(r_{i,t}, y_t)).
  \end{align*}
  \label{def:hedge-mwu-definition}
\end{definition}
Note that $1-\eta \ell(r_{i,t}, y_t)$ is a linear approximation of $\exp{\left(-\eta \ell(r_{i,t}, y_t)\right)}$ around 0.

Unlike the setting of \cite{roughgarden2017online} where using MWU with a proper loss function $\ell$ implies incentive-compatibility, in this setting we do not achieve incentive-compatibility since the normalization would impact the incentive.\footnote{Note that Hedge is not incentive-compatible even in the setting of \cite{roughgarden2017online}.} Indeed, in the MWU algorithm, in round $t$, depending on the outcome $y_t$, the sum of the weights of all experts (the normalization factor) can be different. This will skew the incentive of an expert who wants to report to maximize the expected normalized weight.

\subsubsection{Connection to wagering mechanism}
Toward getting an update rule that is incentive compatible, \cite{freeman2020no} observed a connection between online learning algorithms and \emph{wagering mechanisms}.
In a wagering mechanism, each player reports their prediction about a random outcome and at the same time wagers (bets) a non-negative amount of money on their prediction. Once the outcome is realized, the mechanism will pay each player a payment based on the quality of their prediction and the amount they wagered. 

More concretely, consider the specific setting of wagering mechanisms where there are $K$ fixed players called experts, and there is an unknown Bernoulli outcome $y \in \{0,1\}$. 
Each expert $i \in [K]$ with belief $b_{i} \in [0,1]$ about the probability that $y=1$ wagers $m_{i} >0$ and reports $r_{i} \in [0,1]$ with the goal of maximizing the expected payment from the mechanism.
Note that $r_i$ may or may not be equal to $b_i$.
Once the random outcome $y$ is revealed, the mechanism takes the vector of reports $\boldsymbol{r} = (r_1, \ldots, r_K)$, wagers $\boldsymbol{m} = (m_1, \ldots, m_K)$, and the realization $y$ and outputs a $K$ dimensional vector of $\Gamma(\boldsymbol{r}, \boldsymbol{m}, y) \in \mathbb{R}^K$ where 
$\Gamma_{i}(\boldsymbol{r}, \boldsymbol{m}, y) \in \mathbb{R}$ is the payment that the mechanism will pay to expert $i$.

A mechanism is called incentive-compatible if any expert $i$ with belief $b_i$ strictly maximizes their expected payment by reporting $r_i=b_i$, i.e.,
\begin{align*}
  b_i = \argmax_{r \in \mathcal{R}}\E_{y \sim \mathrm{Bern}(b_i)}\left[\Gamma_{i}\left((r_{i},\boldsymbol{r}_{-i}\right), \boldsymbol{m}, y)\right]
\end{align*}
for any fixed vector $\boldsymbol{r}_{-i}$ of reports of the other experts and vector of wagers $\boldsymbol{m}$. 

The Weighted-Score Wagering Mechanism (WSWM) is an incentive-compatible wagering mechanism defined as follows.
\begin{definition}[\cite{lambert2008self}]
  The \emph{Weighted-Score Wagering Mechanism} is a wagering mechanism that maps any vectors $\boldsymbol{r}$ and $\boldsymbol{m}$ and outcome $y$ to payment $\Gamma = (\Gamma_{1}, \ldots, \Gamma_{K})$, where  
  \begin{align*}
    \Gamma^{\mathrm{WSWM}}_{i}(\boldsymbol{r}, \boldsymbol{m}, y) = m_{i} \left(1 - \ell(r_i,y) + \sum_{j \in [K]} m_j \ell(r_j,y)\right)
  \end{align*}
  is the payment for expert $i$ and $\ell$ is a strictly proper loss function.
\end{definition}
This mechanism has several essential properties \citep{lambert2008self}. First, the mechanism is budget-balanced meaning $\sum_{j}\Gamma_{j}(\boldsymbol{r}, \boldsymbol{m}, y) = \sum_{j}m_j$, and moreover, the payment is non-negative, i.e., $\Gamma_{i}(\boldsymbol{r}, \boldsymbol{m}, y) \geq 0$. 
Observe that designing an incentive-compatible probability-based online learning algorithm can be seen as designing an incentive-compatible wagering mechanism that is budget balanced with non-negative payment as follows.

Consider the probability-based description of any online learning algorithm. This algorithm wants to reallocate $\boldsymbol{\pi}_t$ to $\boldsymbol{\pi}_{t+1}$. To do that, the algorithm at round $t$ asks for the reports of the experts. Once the outcome $y_t$ is revealed, the algorithm uses the reports of the experts $\boldsymbol{r}_t = (r_{t,1}, \ldots, r_{t,K})$, the wagers vector $\boldsymbol{\pi}_{t} = (\pi_{t,1}, \ldots, \pi_{t,K})$, and outcome $y_t$ to set the probability in the next round as
\begin{align*}
  \pi_{t+1,i} = \Gamma_{i}(\boldsymbol{r}_t, \boldsymbol{\pi}_t, y_t).
\end{align*} 
Since the payment $\Gamma_{i}(\boldsymbol{r}_t, \boldsymbol{\pi}_t, y_t)$ is incentive-compatible, experts will report truthfully.

Therefore, among the algorithms in the class of online probability-based update online learning algorithms defined in Definition~\ref{def:prob-based-description}, the only incentive-compatible ones are the ones where their update can be described as a  wagering mechanism update that is non-negative and budget-balanced.

Using a wagering mechanism itself does not imply a no-regret guarantee; however, \cite{freeman2020no} designed a wagering mechanism called Weighted-Score Update (WSU)\footnote{They do not call their update rule as a wagering mechanism, but it can be viewed as a wagering mechanism.} in which the mechanism uses the update
\begin{align*}
  \pi_{t+1,i} &= \Gamma^{\mathrm{WSU}}_{i}(\boldsymbol{r}_t, \eta \boldsymbol{\pi}_t, y_t)\\
              &= \Gamma^{\mathrm{WSWM}}_{i}(\boldsymbol{r}_t, \eta \boldsymbol{\pi}_t, y_t) + \Gamma^{\mathrm{Const}}_i(\boldsymbol{r}_t, (1-\eta) \boldsymbol{\pi}_t, y_t)\\
              &=\Gamma^{\mathrm{WSWM}}_{i}(\boldsymbol{r}_t, \eta \boldsymbol{\pi}_t, y_t) + (1-\eta)\pi_{t,i}
\end{align*} 
for any $i\in [K]$, where $\eta \in (0,0.5)$ and $\Gamma^{\mathrm{Const}}_i$ is simply a mechanism that returns the input wagers. Using the definition of WSWM, this update may be written as 
\begin{align}
  \pi_{t+1, i} = \pi_{t,i}\Biggl(1 - \eta\Biggl(\ell_{t,i} - \sum_{j} \pi_{t,j} \ell_{t,j}\Biggr)\Biggr).
  \label{eq:some-eq47}
\end{align}
\cite{freeman2020no} showed that WSU with the update of form \eqref{eq:some-eq47} can achieve $O(\sqrt{T \ln{K}})$ regret in the full-information setting. 

Interestingly, an apparently unnoticed connection is that the update of form \eqref{eq:some-eq47} recovers the same update as the ML-Prod update of \cite{gaillard2014second} if in all rounds all experts use the same learning rate $\eta$.

\subsection{Existing bandit results}

\cite{freeman2020no} extended their result to the bandit case by designing the Weighted-Score Update with Uniform Exploration (WSU-UX) algorithm described in Algorithm~\ref{alg:WSU-UX}, and they showed a $O(T^{2/3}(K\ln{K})^{1/3})$ upper bound on the regret of this algorithm.
In particular, in their algorithm, they used the common technique of constructing unbiased importance-weighted loss estimates. The algorithm then applies the WSU update on the estimated losses to update the probability distribution. For some technical reasons, they additionally needed to mix the probability distribution over arms ($\boldsymbol{\pi_{t}}$) with a uniform distribution with weight $\gamma \in [0,1]$ to get the probability distribution ($\boldsymbol{\Tilde{\pi}_{t}}$) from which an arm is selected, i.e., $\Tilde{\pi}_{t,i} = (1-\gamma) \pi_{t,i} + \gamma \frac{1}{k}$. The two technical reasons for using $\gamma$ are as follows:
\begin{enumerate}[topsep=0pt,itemsep=0pt]
\item 
  To make sure that after each update, $\pi_{t,i}$ is still a valid probability distribution. 
\item 
  Their regret upper bound can be extremely large in case they do not mix (i.e.~$\gamma = 0$).
\end{enumerate}

Note that mixing with uniform distribution is not for the purpose of getting high probability bounds, as they bound pseudo-regret.\footnote{``Regret'' in this work is actually pseudo-regret, which equals expected regret under oblivious beliefs and outcomes.}

\begin{algorithm}[t]
  \DontPrintSemicolon
  
  \KwInput{$\eta, \gamma \in (0,1/2) \text{ such that } \frac{\eta K}{\gamma}\leq 1/2$, and loss sequence $\ell(x,y)$.}
  Set $\pi_{1,i} = \frac{1}{K}, \forall i \in [K]$
  
  \For{$t \in [T]$}
  {
    The learner chooses expert $I_t$ according to distr. $ \tilde{\pi}_{t,i} = (1 - \gamma) \pi_{t,i} + \frac{\gamma}{K}, \forall i \in [K]$.
    
    Arm $i=I_t$ forms a belief $b_{i,t} \in [0,1]$.
    
    Arm $i=I_t$ reports a report $r_{i,t} \in [0,1]$ with the goal of maximizing $\E_{y_t \sim \mathrm{Bern(b_{i,t})}}[\pi_{t+1,i}]$.
    
    Nature reveals the outcome $y_t \in \{0,1\}$

    The learner computes $\hat{\ell}_{i,t} = \frac{\ell(r_{i,t}, y_t)}{\tilde{\pi}_{i,t}}$ for $i=I_t$ and  $\hat{\ell}_{j,t}=0, \forall j \neq I_t$.
    
    The learner updates $\pi_{t+1,i} = \pi_{t,i} \left(1 - \eta \left(\hat{\ell}_{t,i} - \sum_{j=1}^{K} \pi_{t,j}\hat{\ell}_{t,j}\right)\right)$.

  }
  \caption{WSU-UX \citep{freeman2020no}}
  \label{alg:WSU-UX}
\end{algorithm}

Indeed, they showed for WSU-UX with learning rate $\eta$ and mixing weight $\gamma$,
\begin{align*}
  \mathbb{E}[\regret_T] \leq \gamma T + \frac{\eta K T}{\gamma} + \frac{\ln{K}}{\eta} + 2\eta KT.
\end{align*}
The best choice of $\eta$ and $\gamma$ attains the regret
\begin{align*}
  \mathbb{E}[\regret_T] \leq 2 (4T)^{2/3} (K\ln{K})^{1/3}, 
\end{align*}
which is $O(T^{2/3})$ in terms of $T$. 

It was unclear whether the higher regret was an artifact of their analysis or a limitation of WSU-UX. If there is a tighter analysis of WSU-UX's regret, then for some valid $\gamma, \eta$, we would have
$\mathbb{E}[\regret_T] = o(T^{2/3})$. However, we show that for any valid $\gamma, \eta$, (see Section~\ref{sec:lower_bound_high_level} for a description of valid hyperparamters) for large enough $T$, there exists a loss sequence for which $\mathbb{E}[\regret_T] = \Omega(T^{2/3})$, implying that WSU-UX cannot achieve regret better than $O(T^{2/3})$.

\section{Regret lower bound for WSU-UX}
\subsection{Potential analysis view comparision between EXP3 and WSU-UX}
\label{sec:potential-section}
A potential(-based) analysis is a common way to analyze the regret of online learning algorithms \citep{cesa2003potential}. 
We compare the potential analysis of WSU-UX and EXP3, 
beginning with the full-information variation of each algorithm and then turning to the implications in the bandit setting. 

In the potential analysis of Hedge, for any $i \in [K]$ and $t \in [T]$, we define $\Phi^{\mathrm{HEDGE}}_{t,i}~:=~w_{t,i}$ with $w_{t,i}$ as in Definition~\ref{def:hedge-mwu-definition}. We define $\Phi^{\mathrm{HEDGE}}_{t}~:=~\sum_{j\in [K]}w_{t,j}$. By non-negativity of $w_{t,i}$, we have
\begin{align}
  \frac{1}{\eta} \ln{(\Phi^{\mathrm{HEDGE}}_{T+1,i})}
  \leq \frac{1}{\eta} \ln{(\Phi^{\mathrm{HEDGE}}_{T+1})},
  \label{eq:some-eq49}
\end{align}
where $\eta$ is the learning rate of the algorithm. 
From the LHS and RHS of \eqref{eq:some-eq49}, we can extract the cumulative loss of expert $i$ and the cumulative loss of the learning algorithm respectively. However, we might not be able to exactly extract these two quantities from the potentials as there might be some error terms involved in the extraction process. Indeed, for Hedge, the LHS is exactly the cumulative loss of 
expert $i$; 
however, the RHS can only be upper bounded by the cumulative loss of the learner plus some extra terms:
\begin{align}
  \frac{1}{\eta} \ln{(\Phi^{\mathrm{HEDGE}}_{T+1})} &\leq \sum_{t \in [T]}\sum_{j}\pi_{t,j} \ell_{t,j}  \nonumber\\
  +\underbrace{\frac{\ln{K}}{\eta}}_{\text{exploration term}} 
                                                    &+ \eta \sum_{t \in [T]} \underbrace{\left[\sum_{j}\pi_{t,j} \left(\ell_{t,j}\right)^2\right]}_{\text{Second order error}}.
                                                      \label{eq:some-eq52}
\end{align}
These two terms will appear in the regret analysis. 

On the other hand, note that the WSU update can be written as a linear approximation of the Hedge update at the point $\bar{\ell}_t:=\sum_{j} \pi_{t,j} \ell_{t,j}$ (see Appendix~\ref{app:linear-approx} for details). 
This means that WSU just uses a linear approximation of Hedge when updating the potential. 
This change in potential function will impact the process of extracting the regret from the potential. 

For WSU, the potential is defined as $\Phi^{\mathrm{WSU}}_{t,i}~:=~\pi_{t,i}$ and $\Phi^{\mathrm{WSU}}_{t}~:=~\sum_{j\in [K]}\pi_{t,j}=1$. By non-negativity of $\pi_{t,i}$ we have
\begin{align}
  \frac{1}{\eta} \ln{(\Phi^{\mathrm{WSU}}_{T+1,i})}
  \leq \frac{1}{\eta} \ln{(\Phi^{\mathrm{WSU}}_{T+1})} = 0.
  \label{eq:some-eq51}
\end{align}
Now, the RHS of \eqref{eq:some-eq51} (which is $0$) does not involve any second-order error term. In fact, since WSU is normalized, the RHS does not give us information about the regret.  However, we can extract the difference between the cumulative loss of the algorithm and expert $i$ from the LHS of \eqref{eq:some-eq51}, and this extraction process would lead to a second-order error term. Indeed we have
\begin{align}
  \frac{1}{\eta}\ln{(\Phi^{\mathrm{WSU}}_{T+1,i})} \geq 
  &\sum_{t \in [T]} \left[\sum_{j} \pi_{t,j} \ell_{t,j} - \ell_{t,i}\right]  \nonumber\\
  &-\underbrace{\frac{\ln{K}}{\eta}}_{\substack{\text{exploration} \\ \text{term}}} - \,\eta \sum_{t\in [T]} \underbrace{\left[\sum_{j} \pi_{t,j} \ell_{t,j} - \ell_{t,i}\right]^2}_{\text{Second-order error}}.
  \label{eq:some-eq53}
\end{align}
These two terms in \eqref{eq:some-eq53} will appear in the regret. 

\paragraph{Comparing \eqref{eq:some-eq53} and \eqref{eq:some-eq52}:} Note that the error term in \eqref{eq:some-eq53} is a second-order version of $\left[\sum_{j} \pi_{t,j} \ell_{t,j} - \ell_{t,i}\right]$ for a \textit{fixed} $i$ whereas the second-order term in \eqref{eq:some-eq52}, which is $\left[\sum_{j}\pi_{t,j} \left(\ell_{t,j}\right)^2\right]$, is 
a \textit{weighted average} of $\left(\ell_{t,j}\right)^2$ weighted by $\pi_{t}$.  
\paragraph{Implication for bandit case}
Now, in the bandit case where we use $\hat{\ell}_{t,i}$ to be an unbiased estimated loss for quantity  $\ell_{t,i}$, the 
expectation of the second-order term in \eqref{eq:some-eq52} is $\E\left[ \sum_{j}\pi_{t,j} \left(\hat{\ell}_{t,j}\right)^2 \right] = O(K)$, whereas the expectation of the second-order term in \eqref{eq:some-eq53} is  $\E\left[\left(\sum_{j} \pi_{t,j} \hat{\ell}_{t,j} - \hat{\ell}_{t,i}\right)^2\right]\leq \E\left[\left(\sum_{j} \pi_{t,j} \hat{\ell}_{t,j}\right)^2\right] + \E\left[\left(\hat{\ell}_{t,i}\right)^2\right] =2K + \E\left[\frac{1}{\tilde{\pi}_{t,i}}\right] = O(\frac{K}{\gamma})$. This difference makes the regret bound for WSU-UX larger. 

However, it is not clear whether the potential-based analysis is tight or if there might be a way to get a better regret upper bound. Next, we show our main result, a lower bound demonstrating that it is not possible to get a better upper bound.
\subsection{Lower bound proof}
\label{sec:lower_bound_high_level}

We show that WSU-UX cannot achieve regret better than $\Omega(T^{2/3})$. The following theorem restricts the focus to valid settings of the hyperparameters, which we define after the theorem.

\begin{restatable}{theorem}{StateMainTheorem}
  For any valid set of hyperparameters $(\eta, \gamma)$ there exists $T_0$ such that for any $T\geq T_0$,
  \begin{align*}
    \E[\mathrm{\regret}_T] = \Omega(T^{2/3}).
  \end{align*}
  \label{thm:main-theorem}
\end{restatable}
The notion of valid hyperparameters is taken from the restrictions imposed by \cite{freeman2020no}. In particular, the restrictions are $\frac{\eta K}{\gamma}\leq 1/2$ and $\eta, \gamma \in (0,1/2)$. See the beginning of the Appendix~\ref{app:main-result-high-level} for more information about the restrictions.

Now, we further partition the set of valid hyperparameters $(\eta, \gamma)$ into two cases: 
\begin{itemize}[topsep=0pt,itemsep=0pt,itemindent=0pt,leftmargin=1.5em]
\item the trivial case: $ \eta < T^{-2/3} $ or $\gamma > T^{-1/3}$;
\item the non-trivial case: $ \eta \geq T^{-2/3}$ and $\gamma \leq T^{-1/3}$.
\end{itemize}
In the trivial case, either the learning rate $\eta$ is too small, causing the algorithm to take a long time to concentrate on the optimal expert and incurring a large regret of order $\Omega (T^{2/3})$, or $\gamma$ is so large that the uniform exploration would cause the algorithm to incur a large regret of $\Omega (T^{2/3})$. The proof for the trivial case, along with all other results in this paper, can be found in the appendix. 
For the non-trivial case, we show the following.
\begin{restatable}{theorem}{RestateNonTrivialTheorem}
  For $K=2$ and for any valid set of $(\eta, \gamma)$ in the non-trivial case, there exists $T_0$ such that for any $T\geq T_0$, we have a loss sequence  $\{\ell_t\}_{t=1}^T$ such that
  \begin{align}
    \E[\mathrm{\regret}_T] = \Omega(T^{2/3}).
  \end{align}
  \label{thm:the-t-2-3-theorem}
\end{restatable}

\subsection{High-level proof for non-trivial case}
In this subsection, we give a high-level proof of Theorem~\ref{thm:the-t-2-3-theorem}. We first introduce the following loss sequence.  
\begin{definition}
  For any $T$, we define
  \begin{align*}
    \{\ell_t\}_{t=1}^T = \begin{cases}
      \ell_{t,1}=1, \ell_{t,2}=0& \text{for } 1 \leq t \leq \frac{T}{100} \\
      \ell_{t,1}=0, \ell_{t,2}=1& \text{for } \frac{T}{100} < t \leq T
    \end{cases}.
  \end{align*}
  Moreover, we call the set of rounds $ \{t: 1 \leq t \leq \frac{T}{100}\}$ Phase 1, where only arm~1 incurs loss, and the set of rounds  $\{t: \frac{T}{100} < t \leq T\}$ Phase 2, where only arm~2 incurs loss.	
  \label{df:loss-sequence}
\end{definition}
From now on, by $\{\ell_t\}_{t=1}^T$ we mean the loss sequence defined in Definition~\ref{df:loss-sequence}. Note that in this loss sequence, the best arm is arm 1. 
Our goal is to show that this particular loss sequence forces the algorithm to incur large regret. We do this by decomposing the regret into three terms, as described in Theorem~\ref{theorem:theorem_of-case-2-by-symbols}.
\begin{restatable}{theorem}{RestateThreeTermsLowerBound}
  When running WSU-UX for any valid choice of $(\eta, \gamma)$ in the non-trivial case, there exists $T_0$ such that for any $T \geq T_0$, for loss sequence $\{\ell_t\}_{t=1}^T$,  
  we have for some constants $c_1,c_2,c_3 > 0$,
  \begin{align}
    \E\left[\mathrm{\regret}_T\right] \geq c_1\frac{1}{\eta} + c_2  \frac{\eta TK}{\gamma} + c_3 \gamma T.
    \label{eq:some-eq3}
  \end{align}
  \label{theorem:theorem_of-case-2-by-symbols}
\end{restatable}
Note that Theorem~\ref{theorem:theorem_of-case-2-by-symbols} implies  
$\E\left[\mathrm{\regret}_T\right] = \Omega(T^{2/3})$
as the RHS of \eqref{eq:some-eq3} can be lower bounded by $\Omega(T^{2/3})$.

\subsubsection{Proof of Theorem~\ref{theorem:theorem_of-case-2-by-symbols}}
It remains to show \eqref{eq:some-eq3}. To do that, we first introduce the following key lemma that follows similar steps as Lemma~4.3 of \cite{freeman2020no} but with all the inequalities in the reverse direction, which implies a second-order lower bound.  

\begin{restatable}[Second-order lower bound]{lemma}{RestateSecondOrderLemma}
  For any valid choice of $(\eta, \gamma)$ when running WSU-UX on loss sequence 
  $\{\ell_t\}_{t=1}^T$, 
  we get
  \begin{align}
    \sum_{t=1}^T\sum_{j\in [K]}\pi_{t,j}\hat{\ell}_{t,j} - \sum_{t=1}^T\hat{\ell}_{t,1} 
    \geq\frac{\ln{\pi_{T+1,1}} + \ln{K}}{\eta}  + \frac{\eta}{4} \sum_{t=1}^T(\hat{\ell}_{t,1} - \sum_{j\in [K]}\pi_{t,j}\hat{\ell}_{t,j})^2.
      \label{eq:some-eq4}
  \end{align}
  \label{theorem:lower-bound-theorem}
\end{restatable}

Next, to convert \eqref{eq:some-eq4} to the lower bound in \eqref{eq:some-eq3}, we list three claims that hold when running WSU-UX on the loss sequence in Definition~\ref{df:loss-sequence} given hyperparameters falling in the non-trivial case. Note that each claim corresponds to a term on the right-hand side of \eqref{eq:some-eq3}.
\begin{restatable}[Concentration on best arm at the end]{claim}{RestateClaimOne}
  For large enough $T$, 
  there exists $c_1 >0$ such that
  \begin{align}
    \E\left[\ln{\pi_{T+1,1}} + \ln{K}\right] \geq c_1.
    \label{inequality:inequality-in-lawn-lemma}
  \end{align}
  \label{claim:small-lawn-pi-on-expectation}
\end{restatable}

\begin{restatable}[Second moment lower bound]{claim}{RestateClaimTwo}
  For large enough $T$, there exists $c_2$ such that
  \begin{align*}
    \mathbb{E}\left[\sum_{t=1}^{T} \left(\hat{\ell}_{t,1} - \sum_{j \in [1,2]} \pi_{t,j} \hat{\ell}_{t,j}\right)^2\right] \geq 4 c_2\frac{TK}{\gamma}.
  \end{align*}
  \label{claim:expectation-of-sequared-relative-loss-lower-bound}
\end{restatable}

\begin{restatable}[Bias induced by uniform exploration]{claim}{RestateClaimThree}
  For large enough $T$, there exists $c_3$ such that
  \begin{align*}
    \E\left[
    \sum_{t=1}^{T} \Biggl( \, 
    \sum_{j \in [1,2]}\tilde{\pi}_{t,j} \hat{\ell}_{t,j} 
    - \sum_{j \in [1,2]}\pi_{t,j} \hat{\ell}_{t,j}
    \Biggr)
    \right] 
    \geq c_3 \gamma T .
  \end{align*}
  \label{claim:pi-tilda-expectation-to-pi-expectation-conversion-lemma}
\end{restatable}
Given all three claims and Lemma~\ref{theorem:lower-bound-theorem}, the proof of Theorem~\ref{theorem:theorem_of-case-2-by-symbols} is straightforward. 
In loss sequence $\{\ell_t\}_{t=1}^T$, the best arm is arm 1; therefore, Claim~\ref{claim:pi-tilda-expectation-to-pi-expectation-conversion-lemma} gives
\begin{align}
  \E[\mathrm{\regret}_T] &= \mathbb{E}\left[\sum_{t=1}^{T}\sum_{j \in [K]}\tilde{\pi}_{t,j} \hat{\ell}_{t,j} - \sum_{t=1}^{T} \hat{\ell}_{t,1}\right] \nonumber\\
                         &\geq \mathbb{E}\left[\sum_{t=1}^{T}\sum_{j \in [K]}\pi_{t,j} \hat{\ell}_{t,j} - \sum_{t=1}^{T} \hat{\ell}_{t,1}\right] + c_3\gamma T .
                           \label{eq:some-eq9}
\end{align}
To further lower bound \eqref{eq:some-eq9}, we take the expectation of both sides of \eqref{eq:some-eq4} and use Claims~\ref{claim:small-lawn-pi-on-expectation}~and~\ref{claim:expectation-of-sequared-relative-loss-lower-bound} to get
\begin{align*}
  \E\left[\mathrm{\regret}_T\right] \geq c_1\frac{1}{\eta} + c_2  \frac{\eta TK}{\gamma} + c_3 \gamma T.
\end{align*}

\subsubsection{Subtlety of showing the claims}
The proof of the claims is not straightforward and requires subtle work. At a high level, we designed loss sequence $\{\ell_t\}_{t=1}^T$ such that 
for a constant fraction of rounds $\frac{T}{200} \leq t  \leq \frac{T}{100}$, arm 1 (the best arm) has a cumulative loss linearly worse than arm 2. 
This can be used to show  Claim~\ref{claim:expectation-of-sequared-relative-loss-lower-bound}. However, during these rounds $\frac{T}{200} \leq t  \leq \frac{T}{100}$, arm 1 keeps getting small probability and hence very large estimated loss. Yet, because the algorithm has some uniform exploration, it picks arm 1 frequently. Therefore, at the beginning of round $t \geq \frac{T}{100}$, the probability update is very slow. 
Therefore, it is not obvious whether the algorithm can allow  $\pi_{T,1}$ to recover at the end or not.

The next section gives a technical overview of how we prove the claims. We perform a careful analysis of the probability updates, leveraging a recently shown multiplicative form of Azuma's inequality \citep{kuszmaul2021multiplicative} in some key steps to show that indeed with probability at least $1 - O(\frac{1}{T^2})$, 	$\pi_{T_1+T_2+1,1} \geq 1/4$.
We then show that with high probability $\pi_{T+1,1} \geq 3/4$. This implies Claim~\ref{claim:small-lawn-pi-on-expectation} and also Claim~\ref{claim:pi-tilda-expectation-to-pi-expectation-conversion-lemma}.

\section{Showing the claims}
The proof of Claim~\ref{claim:expectation-of-sequared-relative-loss-lower-bound} is fairly straightforward. The proof requires the following lemma.
\begin{restatable}{lemma}{RestateUsefulLemma}\label{lm:recursive-expectation-lemma}
  In WSU-UX with two arms $(i, \bar{i})$,
  \begin{align}
    \E[\pi_{t+1,i} \mid \mathcal{F}_{t-1}] = (1 - C_{t,i}) \pi_{t,i} + C_{t,i} \, \pi^2_{t,i},
    \label{eq:rec-equality}
  \end{align}
  where $C_{t,i}:= \eta \left(\ell_{t,i} - \ell_{t,\bar{i}}\right)$ and $\F_{t-1}$ is the history up until the end of round $t-1$. 
\end{restatable}
This lemma can be used to show that for $t$ in Phase 1, i.e., $t \leq \frac{T}{100}$, $\pi_{t,1}$ decreases in a multiplicative way and we have
\begin{align*}
  \E[\pi_{t+1,1} ] \leq (1 - \frac{\eta}{2})\,\E[\pi_{t,1}].
\end{align*}
The above inequality can be used to show that for a constant fraction of rounds $\frac{T}{200} \leq t \leq \frac{T}{100}$ we have 
\begin{align}
  \mathbb{E}[\pi_{t,1}] \leq \frac{1}{KT}.
  \label{eq:some-eq48}
\end{align}
The above fact can be further utilized to show that the summation of the second moments of the estimated loss differences in Claim~\ref{claim:expectation-of-sequared-relative-loss-lower-bound} is large, which proves Claim~\ref{claim:expectation-of-sequared-relative-loss-lower-bound}. 

Claims~\ref{claim:small-lawn-pi-on-expectation}~and~\ref{claim:pi-tilda-expectation-to-pi-expectation-conversion-lemma} require more sophisticated techniques. To demonstrate them, we need to analyze the behavior of $\pi_{t,1}$ for $t \in [T]$ when running the algorithm on the loss sequence $\{\ell_t\}_{t=1}^T$. Note that since we are in the bandit case, $\pi_{t,1}$ is a random variable.

Recall Phases 1 and 2 from Definition~\ref{df:loss-sequence}, the definition of the loss sequence. We need to further decompose these phases into multiple sub-phases as defined below. 
\begin{restatable}{definition}{RestatePhaseDefinition}
  Define $T_1 = \frac{1}{100}T, T_2 = \frac{2}{10}T, T_3= \frac{1}{10}T, T_4 = \frac{69}{100}T$, and then define (sub-)phases as follows:
  \begin{itemize}[topsep=0pt,itemsep=2pt,itemindent=0pt,leftmargin=1.5em]
  \item Phase 1: $\mathcal{T}_1 = \{t: 1 \leq t \leq T_1\}$,
  \item Phase 2.1: $\mathcal{T}_2 = \{t: T_1 + 1 \leq t \leq T_1 + T_2\}$,
  \item Phase 2.2: $\mathcal{T}_3 = \{t: T_1 +T_2 + 1 \leq t \leq T_1 + T_2+T_3\}$,
  \item Phase 2.3: $\mathcal{T}_4 = \{t: T_1 + T_2+T_3 +1\leq t \leq T_1 + T_2+T_3+T_4\}$.
  \end{itemize}
  \label{def:phase-definition}
\end{restatable}

Moreover, we define $T'$ and $M$, which are intermediate numbers that are going to be used in our analysis regarding high probability statements in this section. 
\begin{restatable}{definition}{RestateMandTPrime}
  We define $M$ and $T'$ as follows
  \begin{align}
    M &:= {\frac{1}{\ln{2}} \,{\left[ \underbrace{\ln {\left(\frac{2K}{\gamma}\right)}}_{\propto (\ln{T})} +  \underbrace{2(1+\varepsilon_1)(1+\frac{\eta K}{\gamma})\eta T_1}_{\propto (\eta T_1)}\right]}}  \label{eq:some-eq37}\\
    T' &:= \frac{1}{1-\varepsilon_2} \, (\frac{4}{3 - \gamma}) \, \frac{2}{\eta} \, M 
         \label{eq:some-eq13}
  \end{align}
  where $\varepsilon_1 = \sqrt{\frac{6\ln{T}}{\frac{2\gamma}{K}T_1}}$ and $\varepsilon_2 = \sqrt{\frac{4 \ln{T}}{\frac{3-\gamma}{4}T_2}}$.
  \label{def:def-h-prime}
\end{restatable}
Note that since we are in the non-trivial case ($\eta \geq T^{-2/3}$ and $\gamma \leq T^{-1/3}$), as $T$ goes to $\infty$, we have $\varepsilon_1, \varepsilon_2 \rightarrow 0$ and moreover $M \approx c \,\eta T_1$ and $T'~=~c' \frac{M}{\eta}~\approx~c \, c' \, T_1$. Note that for large enough $T$, we have
\begin{align}
  T' \leq T_2.
  \label{eq:some-eq14}
\end{align}
We now give a high-level picture of how we prove Claims~\ref{claim:small-lawn-pi-on-expectation}~and~\ref{claim:pi-tilda-expectation-to-pi-expectation-conversion-lemma}, which also will give more intuition about $M$ and $T'$.

\begin{proof}[Proof Sketch of Claims~\ref{claim:small-lawn-pi-on-expectation}~and~\ref{claim:pi-tilda-expectation-to-pi-expectation-conversion-lemma}]
  Let $\mathcal{E}_1 = \{\pi_{T_1+1,1} \geq 2^{-M}\}$ be the event that arm 1's probability at the end of Phase 1 is not too small, where $M$ is defined in \eqref{eq:some-eq37}. 
  Let $\mathcal{E}_2 = \{\pi_{T_1+T_2+1,1} \geq \frac{1}{4} \}$ be the event that arm 1's probability at the end of Phase 2.2 has recovered to $\frac{1}{4}$. 
  
  First, we show that with high probability, the algorithm does not pull arm 1 too many times in Phase 1. In particular, we show that at the end of Phase 1, the probability $\pi_{T_1+1,1}$ of selecting arm 1 is lower bounded by $2^{-M}$ (hence, $\mathcal{E}_1$ happens). To prove this result, we leverage a recent multiplicative form of Azuma's inequality for martingales \citep{kuszmaul2021multiplicative}. 
  
  The next key step is to show that if $\mathcal{E}_1$ happens, then with high probability $\mathcal{E}_2$ happens. Since we already showed that $\mathcal{E}_1$ happens with high probability, it follows that with high probability we have that $\pi_{T_1 + T_2 + 1,1} \geq 1/4$. Now, to show this key step, we proceed as follows. First, we observe that in Phase 2.1, it is only via pulls of arm 2 that the probability of arm 1 can increase. Moreover, initially, the rate of update of arm 1 is $\frac{\pi_{t+1,1}}{\pi_{t,1}}$, which is very close to 1. Therefore, we first analyze how many pulls of arm 2 suffice for arm 1’s probability to double, and we then analyze how many doublings are needed to satisfy event $\mathcal{E}_2$. Finally, we again use a martingale analysis (multiplicative Azuma) to show that within the rounds of Phase 2.1, the sufficient number of pulls of arm 2 occur with high probability and hence event $\mathcal{E}_2$ happens.
  
  We then show that conditional on event $\mathcal{E}_2$, from Phase 2.2 onwards, the probability $\pi_{t,2}$ goes to zero exponentially quickly as $t$ increments beyond $T_1 + T_2$. Therefore, the probability $\pi_{t,1}$ converges to 1 exponentially quickly. This, combined with Lemma~\ref{lm:recursive-expectation-lemma}, is essentially what is needed to prove Claim~\ref{claim:pi-tilda-expectation-to-pi-expectation-conversion-lemma}. For Claim~\ref{claim:small-lawn-pi-on-expectation}, we use a careful analysis based on Chebyshev’s inequality to show that with probability exponentially close to 1, $\pi_{T,1}$ is at least $3/4$. This is essentially what allows us to control the expected log probability term in Claim~\ref{claim:small-lawn-pi-on-expectation} and hence what allows Claim~\ref{claim:small-lawn-pi-on-expectation} to go through.
\end{proof}

\subsubsection*{Acknowledgements}

AM and NM were supported by NSERC Discovery Grant RGPIN- 2018-03942 and a JP Morgan Faculty Research Award.

\bibliography{wsu-ux_lower_bound}

\appendix

\LinesNumbered
\section{Main result and high-level proof}
\label{app:main-result-high-level}
We give a high-level proof of the main theorem. For the convenience of the reader, we re-state the lemmas along with their proofs.

\StateMainTheorem*

We remind the reader that we restrict our attention to the set of valid hyperparameters: $\frac{\eta K}{\gamma}\leq 1/2$ and $\eta, \gamma \in (0,1/2)$. The notion of valid hyperparameters is taken from the restrictions imposed by \cite{freeman2020no}. We briefly explain some of the restrictions. 

We note that we need $\frac{\eta K}{\gamma}\leq 1$ to make sure that we get a valid probability distribution $\pi_{t+1}$ after updating in each round $t$. Moreover, to get sublinear regret, we need $\gamma \leq 1/2$; otherwise, uniform exploration will pick the suboptimal arm frequently, which can cause linear regret. Since $K\geq 1$, the inequality $\frac{\eta K}{\gamma}\leq 1$ implies $\eta \leq \frac{\gamma}{K} \leq \gamma \leq 1/2$ and therefore $\eta \leq 1/2$.

Our analysis focuses on the restriction $\frac{\eta K }{\gamma} \leq c$ for $c = 1/2$. Thus far, for technical reasons related to certain inequalities, we are not sure whether our particular analysis can be made to go through for a larger $c$ (that is still less than 1). The main concern arises as $c$ gets closer to 1. However, using advanced Taylor approximation-based inequalities, this might be possible.

Now, we further partition the set of valid hyperparameters $(\eta, \gamma)$ into two cases: 
\begin{itemize}
\item The trivial case: $ \eta < T^{-2/3} $ or $\gamma > T^{-1/3}$
\item The non-trivial case: $ \eta \geq T^{-2/3}$ and $\gamma \leq T^{-1/3}$.
\end{itemize} 
The proof strategy for this theorem is that for any valid hyperparameters $(\eta, \gamma)$, we show there exists a loss sequence such that the algorithm will incur an expected regret of $\Omega (T^{2/3})$.

\subsection{Trivial case}
\label{apx:trivial-case}
We show that in the trivial case $\E[\mathrm{\regret}_T] = \Omega(T^{2/3})$. 
\begin{lemma}
  For $K=2$ and for any setting $(\eta, \gamma)$ in the trivial case where $\eta < T^{-2/3}$ or $\gamma > T^{1/3}$,  
  \begin{align*}
    \E[\mathrm{\regret}_T] = \Omega(T^{2/3}).
  \end{align*}
  \label{lemma:trivial-case}
\end{lemma}
We prove this by proving the following Lemma~\ref{thm:theorem-for-small-eta} for the case where $\eta < T^{-2/3}$ and Lemma~\ref{thm:theorem-large-exploration} for the case where $\gamma > T^{-1/3}$.
\begin{lemma}
  For WSU-UX run with any $\eta < T^{-2/3}$ and any $\gamma \geq 0$ there exists a loss sequence $\{\ell_t\}_{t=1}^T$ and $c=\frac{1}{200} >0$ such that
  \begin{align*}
    \E \left[ \regret_T \right] \geq c \, T^{2/3}.
  \end{align*}
  \label{thm:theorem-for-small-eta}
\end{lemma}
Before proving Lemma~\ref{thm:theorem-for-small-eta}, let us restate and prove Lemma~\ref{lm:recursive-expectation-lemma}. We will use this lemma numerous times in the course of proving the main result.

\RestateUsefulLemma*

\begin{proof}[Proof of Lemma~\ref{lm:recursive-expectation-lemma}]
  When we have two arms, the update rule for arm $i$ can be expressed as 
  \begin{align*}
    \pi_{t+1,i} &= \pi_{t,i} \left( 1 - \eta \left( \hat{\ell}_{t,i} - \sum_{j \in \{i,\bar{i}\}} \pi_{t,j} \hat{\ell}_{t,j}\right)\right)\\
		& = \pi_{t,i} - \eta\pi_{t,i}\hat{\ell}_{t,i} + \eta\pi_{t,i} \left(\sum_{j \in \{i,\hat{i}\}}\pi_{t,j}\hat{\ell}_{t,j} \right).
  \end{align*}
  Now, taking the expectation of both sides conditional on the past, we get
  \begin{align*}
    \E\left[\pi_{t+1,i}\mid\mathcal{F}_{t-1}\right] &= \E_{t-1}\left[\pi_{t+1,i}\right] \\
                                                    &=\E_{t-1}\left[\pi_{t,i} - \eta\pi_{t,i}\hat{\ell}_{t,i} + \eta\pi_{t,i}\sum_{j \in \{i,\bar{i}\}}\pi_{t,j}\hat{\ell}_{t,j}\right] \\
                                                    &= \pi_{t,i}  
                                                      - \eta \pi_{t,i} \E_{t-1}[\hat{\ell}_{t,i}]  + \eta \pi_{t,i}^2\E_{t-1} \left[\hat{\ell}_{t,i}\right] + \eta (\pi_{t,i})  (\pi_{t,\bar{i}})  \E_{t-1} \left[\hat{\ell}_{t,\bar{i}}\right]\\
                                                    &= \pi_{t,i} - \eta\pi_{t,i}\ell_{t,i} + \eta \pi_{t,i}^2 \ell_{t,i}  + \eta (\pi_{t,i})(1-\pi_{t,i})\ell_{t,\bar{i}}. 
  \end{align*}
  Now, rearranging the terms, we get
  
  \begin{align*}
    \E_{t}\left[\pi_{t+1,i}\right] &= \left(1 - \eta\ell_{t,i} + \eta \ell_{t,\bar{i}}\right) \pi_{t,i} + \eta \left( \ell_{t,i} - \ell_{t,\bar{i}}\right) \pi_{t,i}^2 \\
                                   &= \left(1-C_{t,i}\right) \pi_{t,i} + (C_{t,i}) \, \pi_{t,i}^2,
  \end{align*} 
  where we recall that $C_{t,i}= \eta \left(\ell_{t,i} - \ell_{t,\bar{i}}\right)$.
  \label{proof:useful-lemma}
\end{proof}

Now, we are ready to prove Lemma~\ref{thm:theorem-for-small-eta}.

\begin{proof}[Proof of Lemma~\ref{thm:theorem-for-small-eta}]
  Consider loss sequence with two arms, $\{\ell_t\}_{t=1}^T$ where $  \ell_{t,1}=0, \ell_{t,2}=1 $ for $1 \leq t \leq T$. In this case, we show that regret simplifies to the number of times we pull arm 2. In particular, 
  \begin{align*}
    \E [\mathrm{\regret_T}] &= \E \left[\sum_{t=1}^T\sum_{j=1}^2\tilde{\pi}_{t,j}\ell_{t,j} - \sum_{t=1}^T\ell_{t,1}\right] \\
                            &= \sum_{t=1}^T\E\left[\tilde{\pi}_{t,2}\right] \ell_{t,2} & (\ell_{t,1}=0, \forall t \in[T])\\
                            &= \sum_{t=1}^T\E[\tilde{\pi}_{t,2}]. & (\ell_{t,2}=1, \forall t \in[T])
  \end{align*}
  Remember that 
  \begin{align}
    \tilde{\pi}_{t,i} = (1 -\gamma) \pi_{t,i} + \frac{\gamma}{2}.
    \label{eq:eq-something}
  \end{align}
  Now, taking the expectation of both sides of \eqref{eq:eq-something} for $i=2$, we get
  \begin{align}
    \sum_{t=1}^T \E[\tilde{\pi}_{t,2}] &= \sum_{t=1}^T \E[(1 -\gamma) \pi_{t,2} + \frac{\gamma}{2} ] \nonumber\\
                                       &= (1 -\gamma) \sum_{t=1}^T \E[ \pi_{t,2}] + \frac{\gamma}{2} T.
                                         \label{eq:some-eq16}
  \end{align}
  Next, we lower bound the first term. Observe that by applying Lemma~\ref{lm:recursive-expectation-lemma} on any round $t>1$, we get
  \begin{align}
    \E [\pi_{t,2}|\mathcal{F}_{t-2}] &= (1 -  \eta) \, \pi_{t-1,2} + \eta \, \pi_{t-1,2}^2 \nonumber\\
                                     &\geq (1 -  \frac{\eta}{2}) \, \pi_{t-1,2}, 
                                       \label{eq:some-eq11}
  \end{align} 
  where the last inequality comes from the fact that $\pi_{t-1,2} \leq 1/2$. This fact is true because we know the initial value for $\pi_{1,2}=\frac{1}{2}$ in the first round, as this quantity can only decrease, $\E[\pi_{t,2}^2] \leq \frac{1}{2}\E[\pi_{t,2}]$. Now, taking expectation of both sides of \eqref{eq:some-eq11}, we get
  \begin{align}
    \E [\pi_{t,2}] \geq (1 -  \frac{\eta}{2}) \E[\pi_{t-1,2}].
    \label{eq:some-eq10}
  \end{align}
  Now, using~\eqref{eq:some-eq10} recursively, we get 
  \begin{align*}
    \E [\pi_{t,2}] &\geq \pi_{1,2} \, (1 -  \frac{\eta}{2})^{t-1} \\
                   &\geq \frac{1}{2} e^{-\eta (t-1)}. & (1 - \frac{\eta}{2} \geq e^{-\eta}, \forall \eta: 0 < \eta \leq 0.5)
  \end{align*}
  Setting $T' = \min \{\lfloor \frac{\ln{100}}{\eta} \rfloor + 1 , T\}\leq T$ , we have 
  \begin{align}
    \sum_{t=1}^T \E[\pi_{t,2}] \geq \sum_{t=1}^{T} \frac{1}{2} e^{- \eta (t-1)} &\geq \sum_{t=1}^{T'}  \frac{1}{2} e^{- \eta (t-1)} \nonumber\\ 
    \geq \sum_{t=1}^{T'} \frac{1}{200} = \frac{T'}{200} 
                                                                                &\geq \frac{1}{200}   
                                                                                  \min\{\frac{\ln{100}}{\eta}, T \}.
                                                                                  \label{eq:another-eq}
  \end{align}
  Now, by using \eqref{eq:another-eq}, we can further lower bound \eqref{eq:some-eq16} to get
  \begin{align*}
    \sum_{t=1}^T \E[\tilde{\pi}_{t,2}] 
    &= (1 -\gamma) \sum_{t=1}^T \E[ \pi_{t,2}] + \frac{\gamma}{2} T \\
    &\geq (1 -\gamma)  \left(\frac{1}{200}\min\left\{\frac{\ln{100}}{\eta}, T \right\}\right) + \gamma\frac{T}{2}   & (\text{from~\eqref{eq:another-eq}})\\
    &\geq \min\left\{ \frac{1}{200} \min \left\{ \frac{\ln{100}}{\eta}, T \right\}, \frac{T}{2}\right\} & (\gamma \alpha + (1-\gamma) \beta \geq \min\{\alpha, \beta\})\\
    &=   \frac{1}{200} \min\left\{\frac{\ln{100}}{\eta}, T \right\} &  \\
    & \geq \frac{1}{200} \min\left\{\frac{1}{\eta}, T \right\} \\
    &\geq \frac{1}{200} \, T^{2/3} & (\eta > T^{-2/3} )\\
    &= c\, T^{2/3}.
  \end{align*} 
\end{proof}
We now present the next lemma.
\begin{lemma}	
  For WSU-UX run with any $ \gamma > T^{-1/3}$ and any $\eta \geq 0$ there exists a loss sequence $\{\ell_t\}_{t=1}^T$ and $c =\frac{1}{2}>0$ such that
  \begin{align*}
    \E \left[ \regret_T \right] \geq c \,
    T^{2/3}.
  \end{align*} 
  \label{thm:theorem-large-exploration}
\end{lemma}
\begin{proof}
  Consider the same loss sequence as Lemma~\ref{thm:theorem-for-small-eta} with two arms: $\{\ell_t\}_{t=1}^T$ where $  \ell_{t,1}=0, \ell_{t,2}=1 $ for $1 \leq t \leq T$. The best arm is arm 1. Since we have a uniform exploration of $\gamma$, in any round $t\in[T]$, we pick arm 2 with probability $\tilde{\pi}_{t,2} = (1-\gamma) \pi_{t,2} + \gamma\frac{1}{2}\geq \gamma\frac{1}{2}$. Therefore, the algorithm incurs at least  $\gamma\frac{1}{2}$ loss for each round. Hence 
  \begin{align*}
    \E \left[\regret_T\right] = \sum_{t=1}^T \E \left[ \tilde{\pi}_{t,2}\right] \geq \sum_{t=1}^T \frac{\gamma}{2} = \frac{\gamma T}{2} > c \, T^{2/3},
  \end{align*} 
  where the last inequality holds because $ \gamma > T^{-1/3}$.
\end{proof}

\subsection{Non-trivial case}
\label{apx:non-trivial-case-subsection}
For the non-trivial case, we have the following theorem as stated in the main text.

\RestateNonTrivialTheorem*

Note that in the non-trivial case, we always consider the loss sequence defined in Definition~\ref{df:loss-sequence}. For the convenience of the reader, we recall that for any $T$, we define $\{\ell_t\}_{t=1}^T$ as 
\begin{align*}
  \{\ell_t\}_{t=1}^T = \begin{cases}
    \ell_{t,1}=1, \ell_{t,2}=0& \text{for } 1 \leq t \leq \frac{T}{100} \\
    \ell_{t,1}=0, \ell_{t,2}=1& \text{for } \frac{T}{100} < t \leq T
  \end{cases}.
\end{align*}

To prove Theorem~\ref{thm:the-t-2-3-theorem}, by Lemma~\ref{lm:lemma15} we first observe that any bound of the form $c_1\frac{1}{\eta} + c_2 \frac{\eta K T}{\gamma} + c_3 \gamma T $ can be lower bounded by $c_4 K^{\frac{1}{3}}T^{\frac{2}{3}}$.
\begin{lemma}
  For any choice of $c_1, c_2, c_3 > 0 $, and any valid choice of $\gamma, \eta >0$, there exists $c_4 > 0 $ such that 
  \begin{align*}
    c_1\frac{1}{\eta} + c_2 \frac{\eta K T}{\gamma} + c_3 \gamma T \geq c_4 K^{\frac{1}{3}}T^{\frac{2}{3}} = \Omega(T^{2/3}).
  \end{align*}
  \label{lm:lemma15}
\end{lemma}

\begin{proof}
  Observe that $c_1\frac{1}{\eta} + c_2 \frac{\eta K T}{\gamma} \geq 2\sqrt{\frac{c_1}{\eta} \, \frac{c_2\eta KT}{\gamma}} $. Therefore, $c_1\frac{1}{\eta} + c_2 \frac{\eta K T}{\gamma} + c_3 \gamma\, T \geq 2\sqrt{\frac{c_1 c_2 KT}{\gamma}} + c_3 \gamma \, T$. Define $f(\gamma) := 2\sqrt{\frac{c_1 c_2 KT}{\gamma}} + c_3 \gamma \, T$ for $\gamma >0$. Note that since $f$ is convex, the minimum is attained at $\gamma^*$ where $f'(\gamma^*) = 0$. It is easy to see that $\gamma^* = c_3^{-\frac{2}{3}} \cdot (\frac{c_1 c_2 K}{T})^\frac{1}{3}$ and therefore, we get
  \begin{align*}
    f(\gamma) &\geq f(\gamma^*) \\
              &= f(c_3^{-\frac{2}{3}} \cdot (\frac{c_1 c_2 K}{T})^\frac{1}{3})\\
              &= 2\sqrt{\frac{c_1 c_2 KT}{c_3^{-\frac{2}{3}} \cdot (\frac{c_1 c_2 K}{T})^\frac{1}{3}}} + c_3^{\frac{1}{3}}(c_1 c_2 K)^{\frac{1}{3}} T^{\frac{2}{3}}\\
              &= 2\sqrt{(c_1 c_2 c_3 K)^{\frac{2}{3}} T^{\frac{4}{3}}} + (c_1 c_2 c_3 K)^{\frac{1}{3}} T^{\frac{2}{3}}\\
              &= 2(c_1 c_2 c_3 K)^{\frac{1}{3}} T^{\frac{2}{3}} + 
		(c_1 c_2 c_3 K)^{\frac{1}{3}} T^{\frac{2}{3}}\\
              &= 3(c_1 c_2 c_3 K)^{\frac{1}{3}} T^{\frac{2}{3}}.
  \end{align*}
  Therefore, there exists $c_4 = 3(c_1 c_2 c_3)^{\frac{1}{3}} > 0$ such that $c_1 \frac{1}{\eta} + c_2 \frac{\eta K T}{\gamma} + c_3 \gamma T \geq c_4 K^{\frac{1}{3}}T^{\frac{2}{3}}$.
\end{proof}

We will then show that for large enough $T$, the regret on this particular loss sequence can be lower bounded as follows.

\RestateThreeTermsLowerBound*

Theorem~\ref{theorem:theorem_of-case-2-by-symbols} along with Lemma~\ref{lm:lemma15} proves Theorem~\ref{thm:the-t-2-3-theorem}.

\subsubsection{Proving Theorem~\ref{theorem:theorem_of-case-2-by-symbols} using claims}
In this section, we show Theorem~\ref{theorem:theorem_of-case-2-by-symbols} given Claim~\ref{claim:small-lawn-pi-on-expectation}, \ref{claim:expectation-of-sequared-relative-loss-lower-bound}, and  \ref{claim:pi-tilda-expectation-to-pi-expectation-conversion-lemma}. We first restate and prove the second-order lower bound lemma.

\RestateSecondOrderLemma*

\begin{proof}[Proof of Lemma~\ref{theorem:lower-bound-theorem}] 
  We can express $\pi_{T+1,1}$ as follows.
  \begin{align*}
    \pi_{T+1,1} = \frac{1}{K }\prod_{t=1}^T (1-\eta(\hat{\ell}_{t,1} - \sum_{j\in [K]}\pi_{t,j}\hat{\ell}_{t,j})).
  \end{align*}
  Taking the logarithm of both sides, we get
  \begin{align}
    \ln{(\pi_{T+1,1})} &= -\ln{K} + \sum_{t=1}^T \ln{\left(1-\eta(\hat{\ell}_{t,1} - \sum_{j\in [K]}\pi_{t,j}\hat{\ell}_{t,j})\right)}.
                         \label{eq:some-eq2}
  \end{align}
  Next, we observe that for any fixed $t \in [T]$, we have $-1\leq\eta(\hat{\ell}_{t,1} - \sum_{j\in [K]}\pi_{t,j}\hat{\ell}_{t,j})\leq 1/2$ (by Lemma~\ref{lm:range-of-relative-estimated-loss} below). Note that for $-1\leq x \leq 1/2$ we have $\ln{(1-x)}\leq -x - x^2/4$ (by Lemma~\ref{lm:log-to-quadratic-upper-bound} below). Therefore, we can show
  \begin{align*}
    \ln{\left(1-\eta(\hat{\ell}_{t,1} - \sum_{j\in [K]}\pi_{t,j}\hat{\ell}_{t,j})\right)} \leq     - \eta(\hat{\ell}_{t,1} - \sum_{j\in [K]}\pi_{t,j}\hat{\ell}_{t,j}) - \frac{\eta^2}{4}(\hat{\ell}_{t,1} - \sum_{j\in [K]}\pi_{t,j}\hat{\ell}_{t,j})^2 
  \end{align*}
  Taking the summation over $T$ rounds and combining with \eqref{eq:some-eq2}, we get
  \begin{align*}
    \ln{\pi_{T+1,1}} \leq - \ln{K}      - \eta\sum_{t=1}^T(\hat{\ell}_{t,1} - \sum_{j\in [K]}\pi_{t,j}\hat{\ell}_{t,j}) - \frac{\eta^2}{4}\sum_{t=1}^T(\hat{\ell}_{t,1} - \sum_{j\in [K]}\pi_{t,j}\hat{\ell}_{t,j})^2 .
  \end{align*}
  Rearranging and dividing by $\eta$, we get
  \begin{align*}
    \sum_{t=1}^T\sum_{j\in [K]}\pi_{t,j}\hat{\ell}_{t,j} - \hat{\ell}_{t,1} &\geq  \frac{\ln{\pi_{T+1,1}} + \ln{K} }{\eta}  + \frac{\eta}{4} \sum_{t=1}^T(\hat{\ell}_{t,1} - \sum_{j\in [K]}\pi_{t,j}\hat{\ell}_{t,j})^2.
  \end{align*}
\end{proof}
The previous proof used the following two simple lemmas.
\begin{lemma}
  $\ln{(1-x)} \leq -x-x^2/4$ when $ -1 \leq x \leq 1/2$.
  \label{lm:log-to-quadratic-upper-bound}
\end{lemma}

\begin{proof}
  Let $f(x) = \ln{(1-x)} + x + x^2/4$. Then $f'(x) = -1/(1-x) + 1 + x/2 = \frac{-x(1+x)}{2-2x}$. Therefore $ f'(x) \geq 0$ for $-1 \leq x \leq 0$, and $f'(x)<0$ for  $0 < x \leq 1/2$. The maximum is attained at $x = 0$. As a result, $f(x) \leq f(0) =  0$ when $-1 < x \leq 1/2$. 
\end{proof}

\begin{lemma}
  For $0 \leq \eta, \gamma \leq 1/2$ where $\frac{\eta K}{\gamma} \leq 1/2$, and loss sequence 
  $\left\{\ell_{t}\right\}_{t=1}^T$, for any round $t \in [T]$, we have
  
  \[-1 \leq \eta  \left(\hat{\ell}_{t,1} - \sum_{j\in \{1,2\}} \pi_{t,j}\hat{\ell}_{t,j}\right) \leq 1/2.\]
  \label{lm:range-of-relative-estimated-loss}
\end{lemma}

\begin{proof}
  For $ 1 \leq t \leq \frac{T}{100}$, we have $\ell_{t,2} = 0$, and hence $\hat{\ell}_{t,2} = 0$. Therefore, we have 
  \begin{align*}
    \eta \left(\hat{\ell}_{t,1} - \sum_{j\in \{1,2\}} \pi_{t,j}\hat{\ell}_{t,j}\right) &= \eta \left(\hat{\ell}_{t,1} - \pi_{t,1}\hat{\ell}_{t,1}\right) \\
                                                                                       &= \eta  \ell_{t,1} \frac{1-\pi_{t,1}}{\tilde{\pi}_{t,1}} \indicator[I_t=1] .
  \end{align*}
  Observe that
  \begin{align*}
    -1 < 0 &\leq \eta  \, \ell_{t,1} \frac{1-\pi_{t,1}}{\tilde{\pi}_{t,1}} \indicator[I_t=1] \\
           &\leq \eta \frac{1}{\tilde{\pi}_{t,1}}\\
           & \leq \eta \frac{K}{\gamma} & (\tilde{\pi}_{t,1} \geq \frac{\gamma}{K})\\
           &\leq 1/2. 
  \end{align*}
  For $\frac{T}{100} < t \leq T$, we have $ \ell_{t,1} = 0$, and hence $\hat{\ell}_{t,1} = 0$. Therefore 
  \begin{align*}
    \eta \, (\hat{\ell}_{t,1} - \sum_{j\in \{1,2\}} \pi_{t,j}\hat{\ell}_{t,j}) &= \eta \, (- \pi_{t,2}\hat{\ell}_{t,2}) \\
                                                                               &= -\eta \, \ell_{t,2} \frac{\pi_{t,2}}{\tilde{\pi}_{t,2}} \indicator[I_t=2]\\
                                                                               &= -\eta \frac{\pi_{t,2}}{\pi_{t,2}(1-\gamma) + \frac{\gamma}{K}} \ell_{t,2}\indicator[I_t=2] .
  \end{align*}
  Now, we can simply show
  \begin{align*}
    1/2 \geq 0 &\geq  -\eta \frac{\pi_{t,2}}{\pi_{t,2}(1-\gamma) + \frac{\gamma}{K}} \ell_{t,2}\,\indicator[I_t=2] \\
               &\geq -\eta \frac{\pi_{t,2}}{\pi_{t,2}(1-\gamma)} \ell_{t,2}\,\indicator[I_t=2]\\
               &\geq - \frac{\eta}{1-\gamma}\\
               &\geq - \frac{\eta}{1/2} &(\gamma\leq 1/2)\\
               &\geq - 1 . &(\eta \leq 1/2)
  \end{align*}
\end{proof}
Then, as mentioned in the main part of the paper, we can use Claims~\ref{claim:small-lawn-pi-on-expectation}, \ref{claim:expectation-of-sequared-relative-loss-lower-bound}, and \ref{claim:pi-tilda-expectation-to-pi-expectation-conversion-lemma} to convert the RHS of \eqref{eq:some-eq4} to the lower bound in \eqref{eq:some-eq3}. 

\section{Proof of claims}
In this part, we prove the claims. 
\label{apx:technical-proofs}

\subsection{Proof of Claim~\ref{claim:expectation-of-sequared-relative-loss-lower-bound}}
\label{apx:toward-claim2}

We first restate Claim~\ref{claim:expectation-of-sequared-relative-loss-lower-bound}. 
\RestateClaimTwo*
In order to prove Claim~\ref{claim:expectation-of-sequared-relative-loss-lower-bound}, we need to introduce the following lemma. 
\begin{lemma}
  When running WSU-UX on loss sequence 	$\left\{\ell_{t}\right\}_{t=1}^T$ and hyperparameter defined in the non-trivial case and for large enough $T$, for $\frac{T_1}{2} \leq t \leq T_1$, we have 
  \begin{align*}
    \mathbb{E}[\pi_{t,1}] \leq \frac{1}{KT}.
  \end{align*}
  \label{lm:1-kt-upper-bound}
\end{lemma}
\begin{proof}
  According to Lemma \ref{lm:recursive-expectation-lemma}, which can be found in Appendix~\ref{proof:useful-lemma}, we have $\E\left[\pi_{t+1,1} \mid \F_{t-1}\right] = (1 - C_t)\,\pi_{t,1} + C_t \, \pi^2_{t,1}$ where $C_t = \eta \, (\ell_{t,1}-\ell_{t,2})$. For $t$ in $\frac{T_1}{2} \leq t \leq T_1$, we have $\pi_{t,1} \leq 1/2$ and $C_t = \eta$. Therefore,
  \begin{align*}
    \E\left[\pi_{t+1,1} \mid \F_{t-1}\right] & = (1 - \eta)\,\pi_{t,1} + \eta\,\pi^2_{t,1} \\
                                             & \leq (1 - \eta)\,\pi_{t,1} + \frac{\eta}{2}\,\pi_{t,1}\\
                                             & = (1 - \frac{\eta}{2})\,\pi_{t,1}.
  \end{align*}
  Taking the expectation over all possible $\F_{t-1}$, we get
  \begin{align*}
    \E[\pi_{t+1,1} ] \leq (1 - \frac{\eta}{2})\,\E[\pi_{t,1}].
  \end{align*}
  Therefore, we have
  \begin{align*}
    \E\left[\pi_{1,\frac{T_1}{2}}\right] = \E\left[\pi_{1,1}\right] \prod_{s=1}^{\lceil\frac{T_1}{2}\rceil-1} \frac{\E\left[\pi_{s+1,1}\right]}{\E\left[\pi_{s,1}\right]} &\leq \frac{1}{2} (1 - \frac{\eta}{2})^{\frac{T_1}{2}}\\
                                                                                                                                                                          & \leq \frac{1}{2} (e^{-\frac{\eta}{2}})^{\frac{T}{200}} & (1+x \leq e^x, T_1 = \frac{T}{100})\\
                                                                                                                                                                          & \leq \frac{1}{2} \exp{\left(-\frac{1}{400}T^{\frac{1}{3}}\right)}. & (\eta \geq T^{-2/3})
  \end{align*}
  Since we are in the non-trivial case and $\eta\geq T^{-2/3}$, we have $\eta \, T \geq T^{1/3}$. Since $\frac{1}{2} e^{-\frac{T^{1/3}}{400}} $ converges to zero exponentially, whereas $\frac{1}{KT}$ convergence to zero at a slower rate, we can say for large enough $T$ that $\frac{1}{2} e^{-\frac{\eta T}{400}}~\leq~\frac{1}{2}e^{-\frac{T^{1/3}}{400}}~\leq~  \frac{1}{KT}$. 
\end{proof}
Now, we are ready to prove Claim~\ref{claim:expectation-of-sequared-relative-loss-lower-bound}.

\begin{proof}[Proof of Claim~\ref{claim:expectation-of-sequared-relative-loss-lower-bound}]
  For $\frac{T_1}{2} \leq t \leq T_1$, we have $\ell_{t,2} = 0$, therefore $\hat{\ell}_{t,2} = 0$. Now we can lower bound $\E\left[\sum_{t=1}^{T} \left(\hat{\ell}_{t,1} - \sum_{j \in [1,2]} \pi_{t,j} \hat{\ell}_{t,j}\right)^2\right] $ as follows:
  \begin{align*}
    \E\left[\sum_{t=1}^{T} \left(\hat{\ell}_{t,1} - \sum_{j \in [1,2]} \pi_{t,j} \hat{\ell}_{t,j}\right)^2\right] &\geq \E\left[\sum_{t=T/100}^{T/200} \left(\hat{\ell}_{t,1} - \sum_{j \in [1,2]} \pi_{t,j} \hat{\ell}_{t,j}\right)^2\right] \\
                                                                                                                  &= \E\left[\sum_{t=T/200}^{T/100} \left(\hat{\ell}_{t,1} - \pi_{t,1} \hat{\ell}_{t,1}\right)^2\right] \\
                                                                                                                  &= \E\left[\sum_{t=T/200}^{T/100} \left(1 - \pi_{t,1}\right)^2 \hat{\ell}_{t,1}^2\right]\\
                                                                                                                  &\geq \E\left[\sum_{t=T/200}^{T/100} (1 - 1/2)^2 \hat{\ell}_{t,1}^2\right] & (\forall t, \frac{T}{200}\leq t\leq \frac{T}{100}: \pi_{t,1}\leq 1/2) \\
                                                                                                                  &= \frac{1}{4} \sum_{t=T/200}^{T/100} \E\left[ {\left(\frac{\ell_{t,1}}{\tilde{\pi}_{t,1}}\indicator[I_t=1]\right)}^2\right] \\
                                                                                                                  &= \frac{1}{4} \sum_{t=T/200}^{T/100} \E\left[\E_{t-1}\left[ {\left(\frac{\ell_{t,1}}{\tilde{\pi}_{t,1}}\right)^2\left(\indicator[I_t=1]\right)^2}\right]\right] \\
                                                                                                                  &= \frac{1}{4} \sum_{t=T/200}^{T/100} \E\left[\left(\frac{1}{\tilde{\pi}_{t,1}}\right)^2 \E_{t-1}\left[ {\left(\indicator[I_t=1]\right)^2}\right]\right] & (\forall t, \frac{T}{200}\leq t\leq \frac{T}{100}: \ell_{t,1} = 1)\\
                                                                                                                  &= \frac{1}{4} \sum_{t=T/200}^{T/100} \E\left[ \frac{1}{\tilde{\pi}_{t,1}}\right] \\ 
                                                                                                                  &\geq \frac{1}{4} \, \sum_{t=T/200}^{T/100} \frac{1}{\E\left[\tilde{\pi}_{t,1}\right]},
  \end{align*}
  where the last inequality comes from applying Jensen's inequality  $\E\left[\frac{1}{X}\right] \geq \frac{1}{\E[X]}$.
  Note that for large enough $T$, for $\frac{T}{200} \leq T \leq \frac{T}{100}$, we have
  \begin{align*}
    \E\left[\tilde{\pi}_{1,t}\right] &= \E\left[\pi_{1,t}(1 - \gamma) + \frac{\gamma}{K}\right] \\
                                     &\leq \E\left[\pi_{1,t}\right] + \frac{\gamma}{K} \\
                                     &\leq \frac{1}{KT} + \frac{\gamma}{K} & (\text{Lemma ~\ref{lm:1-kt-upper-bound}})\\
                                     &\leq \frac{2\gamma}{K}, & (\text{since } \frac{1}{T} \leq \gamma)
  \end{align*}
  where the last inequality comes from the fact that we have $\frac{\eta K}{\gamma} \leq 1/2$ for WSU-UX. This implies $\gamma \geq 2 \eta K \geq 2K T^{-2/3} \geq \frac{1}{T}$, where we use the fact that $\eta \geq T^{-2/3}$.
  As a result, we get 
  \begin{align*}
    \E\left[\sum_{t=1}^{T} \left(\hat{\ell}_{t,1} - \sum_{j \in [1,2]} \pi_{t,j} \hat{\ell}_{t,j}\right)^2\right] &\geq 
                                                                                                                    \frac{1}{4} \sum_{t=T/200}^{T/100} \frac{1}{\mathbb{E}[\tilde{\pi}_{t,1}]}\\
                                                                                                                  &\geq 
                                                                                                                    \frac{1}{4} \sum_{t=T/200}^{T/100} \frac{K}{2\gamma} \\
                                                                                                                  &\geq \frac{1}{4} (\frac{T}{100} - \frac{T}{200})\frac{K}{2\gamma} = \frac{1}{1600}\frac{TK}{\gamma}.
  \end{align*}
  Therefore, for $c_2=\frac{1}{4 \cdot 1600}$, Claim~\ref{claim:expectation-of-sequared-relative-loss-lower-bound} holds. 
\end{proof}
\subsection{Proof sketch of Claims~\ref{claim:small-lawn-pi-on-expectation} and \ref{claim:pi-tilda-expectation-to-pi-expectation-conversion-lemma}}
In this subsection, we prove Claims~\ref{claim:small-lawn-pi-on-expectation} and \ref{claim:pi-tilda-expectation-to-pi-expectation-conversion-lemma} using several technical lemmas without stating their proof. We will prove all the technical lemmas in the next subsection.

We first recall the notion of phases here.

\RestatePhaseDefinition*

We recall the definitions of $M$ and $T'$ as well.

\RestateMandTPrime*

Next, we restate two events $\mathcal{E}_1$ and $\mathcal{E}_2$.
\begin{definition}
  Let  
  \begin{align*}
    \mathcal{E}_1 = \{\pi_{T_1+1,1} \geq 2^{-M}\}
  \end{align*}
  be the event that arm 1's probability at the end of Phase 1 is not too small, where $M$ is defined in \eqref{eq:some-eq37}.
  \label{df:event-1-definition}
\end{definition}

\begin{definition}
  Let 
  $\mathcal{E}_2 = \{\pi_{T_1+T_2+1,1} \geq \frac{1}{4} \}$
  be the event that arm 1's probability at the end of Phase 2.2 has recovered to $\frac{1}{4}$.
  \label{df:event-h-definition}
\end{definition}

Next, we have the following lemma stating that, with high probability, $\pi_{T_1+1,1}$ is not too small. 

\begin{lemma}
  When we run WSU-UX with any valid parameters $\eta, \gamma$ on specific loss sequence $\{\ell_t\}_{t=1}^T$, for any $(\varepsilon , \delta)$, where $ \varepsilon =  \sqrt{\frac{3\ln{\frac{1}{\delta}}}{\frac{2\gamma}{K}T_1}} \in (0,1]$, at the end of phase 1, we have that with probability at least $1-\delta$,
  \begin{align*}
    \pi_{T_1+1, 1} \geq 2^{-\mathcal{M}},
  \end{align*}
  where $\mathcal{M}= {\frac{1}{\ln{2}}{\left(2(1+\varepsilon)(1+\frac{\eta K}{\gamma})\eta T_1 + \ln {\frac{2K}{\gamma}}\right)}}$. In particular, for large enough $T$, by choosing $\delta=\frac{1}{T^2}$, we get $\pi_{T_1+1, 1} \geq 2^{-M}$ where $M$ is defined in \eqref{eq:some-eq37}. 
  \label{lm:hig-prob-lower-bound-on-pi-in-phase-1}
\end{lemma}
We used a recently developed multiplicative form of Azuma's inequality for martingales \citep{kuszmaul2021multiplicative} to show Lemma~\ref{lm:hig-prob-lower-bound-on-pi-in-phase-1}. 
This lemma shows that when $T$ is large enough, with high probability, $\pi_{T_1+1,1}$ does not become too small, i.e., event $\mathcal{E}_1$ happens. Next, we will show that $\pi_{T_1+ T_2+1,1}$ recovers to 1/4 with high probability. To show this, observe that Phase 2.1 is the phase where $\pi_{t,1}$ can start to recover. By each pull of arm 2 in Phase 2.1, the probability $\pi_{t,1}$ increases, whereas by pulls of arm 1, $\pi_{t,1}$ does not change. Hence, we first find an upper bound on the number of pulls of arm 2 needed so that $\pi_{t,1}$ recovers to 1/4. At the beginning of Phase 2.1, after each pull of arm 2, the rate of update $\frac{\pi_{t+1,1}}{\pi_{t,1}} \approx 1+\varepsilon$ is very close to 1. Therefore, we first focus on finding an upper bound on \textit{the number of pulls of arm 2 needed for $\pi_{t,1}$ to double}.
\begin{lemma}
  Consider a round $t_0 > T_1$, where $0<\pi_{t_0,1}\leq \frac{1}{4}$. If arm 2 is pulled $m$ times in rounds $t_1, \ldots, t_m$ where $t_0 \leq t_1< t_2< \ldots < t_m \leq T$ and $m  \geq \frac{2}{\eta}\cdot \frac{1}{1-2\pi_{t_0,1}}$, then we have
  $\pi_{t_m+1, 1} \geq 2 \pi_{t_0,1}$.
  \label{lm:muplitplication-by-2-in-phase2.1}
\end{lemma}
Next, Lemma~\ref{lm:sum-of-pulls-in-phase2.1} is a cumulative version of Lemma~\ref{lm:muplitplication-by-2-in-phase2.1} where we show an upper bound on the total number of pulls of arm 2 needed so that $\pi_{t,1}$ doubles for $M-2$ times.
\begin{lemma}
  Consider a round $t > T_1$ where we have $2^{-m}\leq \pi_{t,1} \leq 2^{-(m-1)}$, where $m \in \mathbb{Z}_{+}$. Then if arm 2 is pulled  $k = \frac{2}{\eta} \, m$ times in rounds $t, t+1, \ldots, T_1$, and round $t'$ denotes the round just after the $k\nth$ pull, then we have 
  $\pi_{t',1}\geq \frac{1}{4}$.
  \label{lm:sum-of-pulls-in-phase2.1}
\end{lemma} 
Lemma~\ref{lm:sum-of-pulls-in-phase2.1} indicates that given event $\mathcal{E}_1$ happened, $k = \frac{2}{\eta} \, M$ pulls of arm 2 suffice to ensure that $\pi_{t,1} \geq 1/4$. 
The next lemma shows that given event $\mathcal{E}_1$ happened, then with high probability, arm 2 is going to be picked in $T_2$ rounds at least $k$ times. This along with Lemma~\ref{lm:sum-of-pulls-in-phase2.1} implies that $\pi_{T_1+T_2+1,1}$ recovers.
\begin{lemma}
  When $T$ is large enough, with probability at least $1-2 \frac{1}{T^2}$, event $\mathcal{E}_2$ happens, i.e.,  
  $\pi_{T_1+T_2+1,1} \geq 1/4$.
  \label{lm:high-prob-for-phase2.1}
\end{lemma}

We next show that given $\mathcal{E}_2$, the expectation of $\pi_{t,2}$ goes to 0 exponentially quickly as $t$ increments beyond $T_1+T_2$.
\begin{lemma}
  Assume event $\mathcal{E}_2$ happens. Define $t_0=T_1 + T_2+1$, and $1 \leq \tau \leq T_3+T_4$, and time step $t = t_0+ \tau$. Then we have
  \begin{align*}
    \E[\pi_{t,2}|\mathcal{E}_2] \leq \frac{3}{4} \exp{(-\frac{\eta}{4} \tau)}.
  \end{align*}
  \label{lm:pi-2-at-the-end-expectation}
\end{lemma}
Now, conditional on event $\mathcal{E}_2$, by using Chebyshev's inequality and Lemma~\ref{lm:pi-2-at-the-end-expectation}, we show that it is very unlikely that $\pi_{T,1}$ is constantly smaller than 1. 
\begin{lemma}
  Conditional on event $\mathcal{E}_2$ defined in Definition~\ref{df:event-h-definition}, we have
  \begin{align*}
    \Pr\left(\pi_{T+1,1} \leq \frac{3}{4} \,\middle|\, \mathcal{E}_2\right) \leq \frac{75}{4} \, e^{-c \eta T}.
  \end{align*}
  \label{lm:prob-of-lowe-pi}
\end{lemma}

Using conditional expectation and Lemma~\ref{lm:prob-of-lowe-pi}, we prove Claim~\ref{claim:small-lawn-pi-on-expectation} in the next subsection.

Next, we introduce the following lemma.
\begin{lemma}
  For large enough $T$, we have  $\E\left[\pi_{t,2}\right]\leq\frac{1}{4}$ for $T_1+T_2+T_3 < t \leq T$.
  \label{lm:dominant-stage}
\end{lemma}
We prove  Claim~\ref{claim:pi-tilda-expectation-to-pi-expectation-conversion-lemma} using Lemma~\ref{lm:dominant-stage} in Appendix~\ref{apx:complete-proof-of-claim-1-3}.

\subsection{Complete proof of Claims~\ref{claim:small-lawn-pi-on-expectation}~and~\ref{claim:pi-tilda-expectation-to-pi-expectation-conversion-lemma}}
\label{apx:complete-proof-of-claim-1-3}

\subsubsection{High probability lemma for event $\mathcal{E}_1$}
\label{apx:high-probability-event-1}

\begin{proof}[Proof of Lemma~\ref{lm:hig-prob-lower-bound-on-pi-in-phase-1}]
  We define $\tau$ to be the largest $t$ in $1 \leq t \leq T_1$ such that $\pi_{t,1} > \frac{\gamma}{K}$. Since $\pi_{t,1}$ are random variables that depend on internal random bits of the algorithm, $\tau$ is also a random variable. 
  Observe that for the loss sequence $\{\ell_t\}_{t=1}^T$ we are considering, for all $1 \leq t\leq T_1+1$ we have $\pi_{t,1} \geq \pi_{t+1,1}$. This implies that 
  \begin{align}
    \pi_{t,1} &> \frac{\gamma}{K} &\forall t: 1 &\leq t \leq \tau \label{eq:some-eq17}\\
    \pi_{t,1} &\leq \frac{\gamma}{K} &\forall t: \tau +1 &\leq t \leq T_1 . \label{eq:some-eq35}
  \end{align}	
  Using $\tau$, we can express $\pi_{T_1+1, 1}$ as follows\footnote{For convention, the product over an empty set is assumed to be 1. e.g. if $\tau=1$, then the first term in the right-hand side of~\eqref{eq:some-eq18} which is a product over the empty set is assumed to be 1. Similarly, if $\tau=T_1$, the third term in the right-hand side of~\eqref{eq:some-eq18} is assumed to be 1.}:
  \begin{align}
    \pi_{T_1+1, 1} &= \pi_{1,1} \prod_{s=1}^{T_1}\frac{\pi_{s+1,1}}{\pi_{s,1}} \nonumber\\
                   &= \underbrace{\left(\pi_{1,1} \prod_{s=1}^{\tau-1}\frac{\pi_{s+1,1}}{\pi_{s,1}} \right)}_{\text{first term}}  \underbrace{\left(\frac{\pi_{\tau+1,1}}{\pi_{\tau,1}}\right)}_{\text{second term}}  \underbrace{\left(\prod_{s=\tau+1}^{T_1}\frac{\pi_{s+1,1}}{\pi_{s,1}}\right)}_{\text{third term}} .
                     \label{eq:some-eq18}
  \end{align}
  Clearly, the first term in \eqref{eq:some-eq18} can be lower bounded as
  \begin{align}
    \pi_{1,1} \prod_{s=1}^{\tau-1}\frac{\pi_{s+1,1}}{\pi_{s,1}} = \pi_{\tau-1, 1}  >  \frac{\gamma}{K},
    \label{eq:some-eq19}
  \end{align}
  where we used \eqref{eq:some-eq17}. 
  Next, observe that for rounds $1\leq t\leq T_1$, in phase one, we can simply lower bound $\frac{\pi_{s+1,1}}{\pi_{s,1}}$ as follows:
  \begin{align}
    \frac{\pi_{s+1,1}}{\pi_{s,1}} &= 1 - \eta \Big(\hat{\ell}_{s,1} - \sum_{j \in [2]}\pi_{s,j}\hat{\ell}_{s,j}\Big) \nonumber\\
                                  &=1 - \eta \frac{1-\pi_{s,1}}{(1-\gamma)\pi_{s,1} + \frac{\gamma}{K}} \indicator[I_s = 1]\nonumber\\
                                  &\geq 1 - \eta \frac{1-0}{0 + \frac{\gamma}{K}} \indicator[I_s = 1] & (\pi_{s,1}\geq 0)\nonumber\\
                                  &=1 -  \frac{\eta K}{\gamma}\indicator[I_s = 1]\nonumber\\
                                  &=(1 -  \frac{\eta K}{\gamma})^{\indicator[I_s=1]}.
                                    \label{eq:eq-telescope-update}
  \end{align}
  Therefore, the second term can be lower bounded by 
  \begin{align}
    \underbrace{\frac{\pi_{\tau+1,1}}{\pi_{\tau,1}}}_{\text{second term}} \geq (1 -  \frac{\eta K}{\gamma})^{\indicator[I_{\tau}=1]} \geq 1/2,
    \label{eq:some-eq25}
  \end{align}
  since $\frac{\eta K}{\gamma}\leq 1/2$ in WSU-UX. Moreover, we have
  \begin{align}
    \underbrace{\prod_{s=\tau+1}^{T_1}\frac{\pi_{s+1,1}}{\pi_{s,1}}}_{\text{third term}} &\geq 
                                                                                           \prod_{s=\tau+1}^{T_1}(1 -  \frac{\eta K}{\gamma})^{\indicator[I_s=1]} & (\text{by~\ref{eq:eq-telescope-update}} ) \nonumber\\
                                                                                         &= (1 -  \frac{\eta K}{\gamma})^{\sum_{s=\tau+1}^{T_1}\indicator[I_s=1]} .
                                                                                           \label{eq:some-eq20}
  \end{align}  
  Therefore, by plugging \eqref{eq:some-eq19}, \eqref{eq:some-eq25}, and  \eqref{eq:some-eq20} into the right hand side of \eqref{eq:some-eq18} we obtain
  \begin{align}
    \pi_{T_1+1,1 }\geq \frac{\gamma}{2K} (1 -  \frac{\eta K}{\gamma})^{\sum_{s=\tau+1}^{T_1}\indicator[I_s=1]},
    \label{eq:some-eq26}
  \end{align}
  where the right-hand side is a random variable that depends on $\tau$. 
  Note that $\indicator\left[\tau = t' | \mathcal{F}_{t'}\right]$ is measurable, meaning that in round $t'$, given access to the past history, we can deterministically tell whether $\tau = t'$ or not.
  Next, we will show that for any possible value $t'$ that $\tau$ can take, and particular history $\F_{t'}$ up until the end of round $t'$ for which $\tau=t'$,\footnote{Observe that this $t'$ is well defined since $\F_{t'-1}$, the history up until round $t'-1$, is enough to determine whether $\tau=t'$ or $\tau \neq t'$.} 
  there is a suitable upper bound on the following term
  \begin{align}
    \sum_{s=t'+1}^{T_1}\indicator[I_s=1 \mid \F_{t'}] 
    \label{eq:some-eq24}
  \end{align}
  that holds with probability at least $1-\delta$.  This implies a uniform high probability lower bound on \eqref{eq:some-eq26}, which completes the proof.

  It remains to show this uniform upper bound for \eqref{eq:some-eq24}.
  In order to show this, we set up a martingale. In particular, we define $I_{t,i} := \indicator[I_t=i]$. Consider any fixed history up until round $t'$ denoted by $\F_{t'}$ such that $\tau=t'$. Then for any $t$ where $t'+1\leq t\leq T_1$, we have
  \begin{align}
    \E \left[ I_{t,1} \,\middle|\, \F_{t-1} \right] &= \E \left[(1-\gamma)\pi_{t,1} + \frac{\gamma}{K} \,\middle|\, \F_{t-1}\right] & (\text{Uniform Exploration by WSU-UX})\nonumber\\
                                                    &\leq \E \left[\pi_{t,1} + \frac{\gamma}{K} \,\middle|\, \F_{t-1}\right]\nonumber\\
                                                    &\leq \frac{2\gamma}{K} =:q, & (\text{from}~\ref{eq:some-eq35})
                                                                                   \label{eq:some-eq23}
  \end{align}
  where the last inequality comes from \eqref{eq:some-eq17}. Define $Z_{t'}:=0$ and for any $t$ where $t' + 1 \leq t\leq T_1$, let $W_t := I_{t,1} - \E [ I_{t,1} \mid \F_{t-1} ]$ and $Z_t := \sum_{s=t'+1}^t W_s$. Observe that for $t$ in   $t' +1 \leq t \leq T_1$, we have 
  \begin{align*}
    \E[Z_t \mid \F_{t-1}] = Z_{t-1} + \E[W_t \mid \F_{t-1}] = Z_{t-1} ;
  \end{align*} 
  therefore $(Z_t)_{t \in \{ t', \ldots, T_1\}}$ is a martingale. Moreover, for $t$ in  $t' +1 \leq t \leq T_1$,  we have $Z_t - Z_{t-1} = W_t \in [-A_t, B_t]$ for $A_t = \E \left[ I_{t,1}\mid \F_{t-1} \right]$ and $B_t = 1 - A_t$, simply because $I_{t,1} \in \{0, 1\}$. Observe that we have
  \begin{align*}
    \sum_{t=t'+1}^{T_1} A_t = \sum_{t=t'+1}^{T_1} \E \left[ I_{t,1} \,\middle|\, \F_{t-1} \right] \leq \sum_{t=t'+1}^{T_1} q \leq \sum_{t=1}^{T_1} q = q T_1 = \frac{2\gamma}{K}T_1=: \mu ,
  \end{align*}
  where the first inequality comes from \eqref{eq:some-eq23}. We define $c:= A_t + B_t = 1$ and apply Theorem~10 of~\cite{kuszmaul2021multiplicative} to get that for all $\varepsilon > 0$, 
  \begin{align*}
    \Pr\left(Z_{T_1} - Z_{t'} \geq \varepsilon \mu \,\middle|\, \F_{t'} \right) \leq \exp \left( -\frac{\varepsilon^2 \mu}{(2 + \varepsilon) c} \right) .
  \end{align*}
  Using our definition of the martingale sequence, 
  and noting that $c=1$, we get
  \begin{align*}
    \Pr \left( \sum_{t=t'+1}^{T_1} I_{t,1} \geq \sum_{t=t'+1}^{T_1} \E \left[ I_{t,1} \,\middle|\, \F_{t-1} \right]  + \varepsilon \mu \,\middle|\, \F_{t'}\right) 
    \leq \exp \left( -\frac{\mu  \varepsilon^2}{2 + \varepsilon} \right) .
  \end{align*}
  Using $\sum_{t=t'+1}^{T_1} \E \left[ I_{t,1}\mid \F_{t-1} \right] \leq \mu$ and by imposing the restriction $\varepsilon \leq 1$, we have for $\varepsilon \in (0, 1]$,
  \begin{align*}
    \Pr \left( \sum_{t=t'+1}^{T_1} I_{t,1} \geq (1 + \varepsilon) \mu \,\middle|\, \F_{t'}\right) 
    \leq \exp \left( -\frac{\mu \varepsilon^2}{3} \right) .
  \end{align*}
  Equivalently, conditional on $\F_{t'}$, with probability at least $1-\delta$, where $ \varepsilon =  \sqrt{\frac{3\ln{\frac{1}{\delta}}}{\frac{2\gamma}{K}T_1}} \in (0,1]$,  we have 
  \begin{align}
    \sum_{t=t'+1}^{T_1} I_{t,1} &\leq (1+\varepsilon) \mu \nonumber\\
                                &=(1+\varepsilon) \frac{2\gamma}{K}T_1. 
                                  \label{eq:high_prob_phase1}
  \end{align}
  This is the suitable upper bound we wanted for the quantity in \eqref{eq:some-eq24}. In particular, we have 
  \begin{align}
    (1 -  \frac{\eta K}{\gamma})^{\left[\sum_{t=t'+1}^{T_1} I_t\right]} &\geq  \Big( e^{(-\frac{\eta K}{\gamma})(1+\frac{\eta K}{\gamma})} \Big)^{\left[\sum_{t=t'+1}^{T_1} I_t\right]} & (1-x \geq e^{(-x - x^2)} \text{ for } 0 < x \leq \frac{1}{2}) \nonumber\\
                                                                        &\geq   \Big( e^{(-\frac{\eta K}{\gamma})(1+\frac{\eta K}{\gamma})} \Big)^{(1+\varepsilon) \frac{2\gamma}{K}T_1}  &\text{from }\eqref{eq:high_prob_phase1} \nonumber\\
                                                                        &=  \exp{\Bigg(- 2(1+\varepsilon) (1+\frac{\eta K }{\gamma}) \eta T_1\Bigg)}.
                                                                          \label{eq:some-eq21}
  \end{align}
  Combining this lower bound with \eqref{eq:some-eq26}, we get the following statement.
  
  For any possible value for $\tau$ and any fixed $\F_{t'}$ satisying $\tau=t'$, for any $\varepsilon = \sqrt{\frac{3 \ln{\frac{1}{\delta}}}{\frac{2\gamma}{K}T_1}} \in (0,1]$, with probability at least $1-\delta$, we have 
  \begin{align}
    \pi_{T_1+1, 1}  = \frac{\gamma}{2K} (1 -  \frac{\eta K}{\gamma})^{\left[\sum_{s=\tau+1}^{T_1}\indicator[I_s=1]\right]} &\geq \frac{ \gamma}{2K} \exp{\Bigg(- 2(1+\varepsilon) (1+\frac{\eta K }{\gamma}) \eta T_1\Bigg)} \nonumber \\
                                                                                                                           &= \exp{\Bigg(- \bigg(2(1+\varepsilon)(1+\frac{\eta K}{\gamma})\eta T_1 + \ln {\frac{2K}{\gamma}}\bigg)\Bigg)} \nonumber\\
                                                                                                                           &= 2^{-\mathcal{M}}.
                                                                                                                             \label{eq:some-eq22}
  \end{align}
  Since \eqref{eq:some-eq22} holds true for any possible value of $\tau$ and any fixed $\F_{t'}$ where $\tau=t'$, it holds true in general. 
  
  Moreover, note that when $T$ is large enough, we can choose $\delta = \frac{1}{T^2}$ since $\varepsilon = \sqrt{\frac{6\ln{T}}{\frac{2\gamma}{K}T_1}} \in (0,1]$ for large enough $T$. As a result, we get
  \begin{align*}
    \pi_{T_1+1,1} \geq 2^{-M},
  \end{align*}
  where $M$ is defined in \eqref{eq:some-eq37}.
\end{proof}
\subsubsection{Upper bound on the number of pulls of arm 2}

\begin{proof}[Proof of Lemma~\ref{lm:muplitplication-by-2-in-phase2.1}]
  We can lower bound $\pi_{t_m+1, 1}$ as
  \begin{align*}
    \pi_{t_m+1, 1} = 1-\pi_{t_m+1,2} &= 1 - \pi_{t_0,2} \prod_{s=t_0}^{t_m} 
                                       \frac{\pi_{s+1, 1}}{\pi_{s,1}}\\
                                     &= 1 - \pi_{t_0,2} \prod_{s=t_0}^{t_m} 
                                       \big( 1 - \eta \frac{1-\pi_{s,2}}{(1-\gamma)\pi_{s,2} + \frac{\gamma}{K}}\indicator[I_s=2] \big)\\
                                     &\geq  1 - \pi_{t_0,2} \prod_{s=t_0}^{t_m} 
                                       \big( 1 - \eta \frac{\pi_{s,1}}{1} \indicator[I_s=2] \big) & \big((1-\gamma)\pi_{s,2} + \frac{\gamma}{K}\leq 1\big)\\
                                     &\geq  1 - \pi_{t_0,2} \prod_{s=t_0}^{t_m} 
                                       \big( 1 - \eta \pi_{t_0,1} \indicator[I_s=2] \big) & \big(\pi_{s,1} \geq \pi_{t_0,1}\big) \\
                                     &=  1 - \pi_{t_0,2} \prod_{s=t_0}^{t_m} 
                                       \big( 1 - \eta \pi_{t_0,1}\big)^{\indicator[I_s=2]} \\
                                     & = 1 - \pi_{t_0,2} 
                                       \big( 1 - \eta \pi_{t_0,1}\big)^{m} \\
                                     &\geq  1 -  
                                       \big( 1 - \eta \pi_{t_0,1}\big)^{m} . & (\pi_{t_0,2} \leq 1)
  \end{align*}
  Moreover, 
  \begin{align*}
    (1 - \eta \pi_{t_0,1})^m &\leq \exp{(-\eta \pi_{t_0,1} m)} & (1-x \leq \exp(x) ,\forall x \in \mathbb{R})\\
                             &\leq \exp{\left(\frac{-2\pi_{t_0,1}}{1-2\pi_{t_0,1}}\right)} & (m \geq \frac{2}{\eta \, (1-2\pi_{t_0,1})})\\
                             &\leq 1 - 2\pi_{t_0,1} , & (e^{\frac{-2x}{1-2x}}\leq 1-2x, \forall x \in (0,\frac{1}{2}])  
  \end{align*}
  which means 
  \begin{align*}
    \pi_{t_m+1, 1} \geq 1 - \big( 1 - \eta \pi_{t_0,1}\big)^{m} \geq 2 \pi_{t_0,1} .
  \end{align*}
\end{proof}
\begin{proof}[Proof of Lemma~\ref{lm:sum-of-pulls-in-phase2.1}]
  Note that if $m\leq 2$, then we already have $\pi_{t,1} \geq \frac{1}{4}$, and since $t \geq T_1$, any pulls of arm 2 only increase $\pi_{t,1}$. Hence, after $k$ pulls  we have $\pi_{t',1}\geq \pi_{t,1} \geq \frac{1}{4}$.
  
  Consider the case where $m\geq 3$. We have $2^{-m}\leq\pi_{t,1} \leq 2^{m-1}$. Now we want to upper bound the number of pulls it takes so that $1/4 \leq \pi_{t',1} \leq 1/2$. Suppose we require $k_1$ pulls for the first doubling of $\pi$, $k_2$ for the second doubling, and so forth. This means we need $k=\sum_{i=1}^{m-2} k_i$ pulls before we get  $1/4 \leq \pi_{t',1} \leq 1/2$. Next, we upper bound each $k_i$. To do this, we denote all the rounds after $t$, in which we pull arm 2 as follows
  \begin{align*}
    \underbrace{t_{1}^{(1)}, t_{2}^{(1)}, \ldots,  t_{k_1}^{(1)}}_{\text{rounds before 1st doubling}} \underbrace{t_{1}^{(2)}, \ldots,  t_{k_2}^{(2)}}_{\text{rounds before 2nd doubling}} \ldots
    \underbrace{t_{1}^{(i)}, t_{2}^{(i)}, \ldots  t_{k_i}^{(i)}}_{\text{rounds before i}^{th}\text{ doubling}} \ldots
    \underbrace{t_{1}^{(m_2-2)}, \ldots,  t_{k_{m_2-2}}^{(m_2-2)}}_{\text{rounds before $(m-2)$ doubling}} , \underbrace{t_{k}}_{1/4 \leq \pi_{t_{k},1} }
  \end{align*}
  Now we can upper bound each $k_i$ as follows:	
  \begin{align*}
    k_i &\leq \frac{2}{\eta} \frac{1}{1 - 2 \pi_{t_1}^{(i)}} & (\text{Lemma~\ref{lm:muplitplication-by-2-in-phase2.1}}) \\
        &\leq \frac{2}{\eta} \frac{1}{1 - 2^{-m+i+1}} \,\, . & (\pi_{t_1}^{(i)} \leq 2^{-m+i})
  \end{align*}
  To see why the first inequality holds, observe that we start from round ${t_1}^{(i)}$, and Lemma~\ref{lm:muplitplication-by-2-in-phase2.1} has an upper bound on the number of pulls needed to get doubled. 
  
  Therefore, we get
  \begin{align}
    \sum_{i=1}^{m-2} k_i \leq \sum_{i=1}^{m-2}  \frac{2}{\eta} \frac{1}{1 - 2^{-m+i+1}} &=  \frac{2}{\eta}  \sum_{i=1}^{m-2}  \frac{1}{1 - 2^{-i}} \,\, .
                                                                                          \label{eq:summation-of-pulls}
  \end{align}
  Now, observe that
  \begin{align*}
    \frac{1}{1 - 2^{-i}} &= 1 + \frac{1}{2^i - 1}\\
                         &\leq 1 + \frac{1}{2^{i-1}} \,\, . & (2^i - 1 \geq 2^{i-1} \text{ for } i \geq 1)
  \end{align*}
  Therefore, we get
  \begin{align*}
    \sum_{i=1}^{m-2} k_i &\leq \frac{2}{\eta} \sum_{i=1}^{m-2} \left( 1 + \frac{1}{2^{i-1}} \right) \\
                         &\leq \frac{2}{\eta} \left[ m-2 +  \sum_{i=0}^{m-2} \frac{1}{2^{i}}\right]\\
                         &\leq \frac{2}{\eta} \left[ m-2 +  2 \right] &(\text{geometric series})\\
                         &= \frac{2}{\eta} m .
  \end{align*}
\end{proof}
\subsubsection{High probability lemma for event $\mathcal{E}_2$}
\label{apx:high-prob-for-event-2}
\begin{proof}[Proof of Lemma~\ref{lm:high-prob-for-phase2.1}]
  Recall the definition of $\mathcal{E}_1$ from Definition~\ref{df:event-1-definition} and $\mathcal{E}_2$ from Definition~\ref{df:event-h-definition} as follows:
  \begin{align*}
    \mathcal{E}_1 &= \left\{\pi_{T_1+1,1} \geq 2^{-M}\right\}\\
    \mathcal{E}_2 &= \left\{\pi_{T_1+T_2+1,1} \geq \frac{1}{4}\right\} ,
  \end{align*}
  where $M = {\frac{1}{\ln{2}}{\left(2(1+\varepsilon_1)(1+\frac{\eta K}{\gamma})\eta T_1 + \ln {\frac{2K}{\gamma}}\right)}}$. 
  We then show the following two statements.
  \begin{enumerate}[label=(\alph*)]
  \item For $\delta=\frac{1}{T^2}$, we prove that with probability at least $1-\delta_1$, $\mathcal{E}_1$ happens, i.e., $\Pr(\mathcal{E}_1)\geq 1-\delta_1$
  \item  Given $\mathcal{E}_1$ happened, for $\delta_2=\frac{1}{T^2}$, we prove that with probability at least $1-\delta_2$, $\mathcal{E}_2$ happens, i.e., $\Pr(\mathcal{E}_2\mid\mathcal{E}_1 )\geq 1-\delta_2$.
  \end{enumerate}
  Having both (a) and (b) implies that the lemma holds true, as 
  \begin{align*}
    \Pr(\mathcal{E}_2) \geq 	\Pr(\mathcal{E}_1 \text{ and }\mathcal{E}_2) = \Pr(\mathcal{E}_1) \, \Pr(\mathcal{E}_2\mid\mathcal{E}_1) &= (1-\delta_1)\,(1-\delta_2) \\
                                                                                                                                         &\geq 1 - \delta_1 -\delta_2 = 1 - \frac{2}{T^2}.
  \end{align*}
  Now, we prove (a) and (b).
  \begin{enumerate}[align=left]
  \item[Proof of (a)]
    
    By Lemma~\ref{lm:hig-prob-lower-bound-on-pi-in-phase-1}, with probability at least $1-\frac{1}{T^2}$, we have
    \begin{align*}
      \pi_{T_1+1, 1} &\geq 2^{-M},
    \end{align*}
    for $M={\frac{1}{\ln{2}}{\left(2(1+\varepsilon_1)(1+\frac{\eta K}{\gamma})\eta T_1 + \ln {\frac{2K}{\gamma}}\right)}}$.
    
  \item[Proof of (b)]
    
    We show that for any history $\F_{T_1}$ such that 
    $\mathcal{E}_1$ happened, we have 
    \begin{align}
      \Pr (\mathcal{E}_2 \mid F_{T_1}) \geq 1 - \delta_2.
      \label{eq:some-eq36}
    \end{align}
    This implies that $\Pr (\mathcal{E}_2 \mid \mathcal{E}_1) \geq 1 - \delta_2$.
    
    Consider a fixed history $\F_{T_1}$ such that 
    $\mathcal{E}_1$ happened. Event $\mathcal{E}_1$ implies that for some $M' \leq M$, we have 
    \begin{align*}
      2^{-M'} \leq \pi_{T_1+1,1} \leq  2^{-(M'-1)} . 
    \end{align*}
    
    Now, Lemma~\ref{lm:sum-of-pulls-in-phase2.1} states that $ \Gamma=\frac{2}{\eta} M'$ pulls is sufficient to get
    \begin{align}
      \pi_{t',1} \geq \frac{1}{4},
    \end{align}
    where $t'$ is round number after $\Gamma$-th pull. We define 
    \begin{align*}
      X_t := \indicator\left[I_t=2 \text{ or } \pi_{t,1} \geq \frac{1}{4}\right].
    \end{align*}
    Next, observe that if 
    \begin{align}
      \sum_{t \in \text{phase 2.1}} X_t = \sum_{t=T_1+1}^{T_1 + T_2 } X_t \geq \frac{2}{\eta} M',
      \label{eq:some-eq27}
    \end{align}
    then this implies that $\pi_{T_1+T_2+1} \geq 1/4$, (i.e.~$\mathcal{E}_2$ happens.) To see why, note that if for any for $t$ in $T_1+1\leq t \leq T_1+T_2$, we have $\pi_{t,1}\geq 1/4$, this implies $\pi_{T_1+T_2+1} \geq 1/4$ since $\pi_{t,1}$ can only increase in phase 2.1. If for all $t$ in $T_1+1\leq t \leq T_1+T_2$, we have $\pi_{t,1} < 1/4$, then \eqref{eq:some-eq27} implies that 
    \begin{align*}
      \sum_{t=T_1+1}^{T_1 + T_2 } X_t = \sum_{t=T_1+1}^{T_1+T_2} \indicator[I_t=2] \geq \frac{2}{\gamma}M'.
    \end{align*}
    Therefore, Lemma~\ref{lm:sum-of-pulls-in-phase2.1} implies that $\pi_{T_1+T_2+1} \geq 1/4$. 
    
    Now it remains to show that with probability at least $1-\frac{1}{T^2}$, \eqref{eq:some-eq27} happens. We use a martingale concentration argument to show this. Indeed, we define $Z_{T_1}=0$ and for any $t$ in $T_1 +1 \leq t \leq T_1+T_2$, we define $W_t := X_t - \E[X_t \mid \F_{t-1}]$ and $Z_t:= \sum_{s=T_1+1}^t W_s$. Observe that 
    \begin{align*}
      \E[Z_t \mid \F_{t-1}] = Z_{t-1} + \E[W_t \mid \F_{t-1}] = Z_{t-1},
    \end{align*} and hence $(Z_t)_{t \in \{T_1,\ldots, T\}}$ is a martingale. Since $X_t \in \{0, 1\}$, we have $Z_t - Z_{t-1} = W_t \in [-A_t, B_t]$ for $A_t = \E \left[ X_t\mid \F_{t-1} \right]$ and $B_t = 1 - A_t$. Consequently, we have $A_t + B_t = 1 := c$ for all $t\geq t'$. Define $q:= \frac{3-\gamma}{4}$. 
    
    Observe that we have $\E \left[X_t\mid \F_{t-1} \right] \geq q$. It is because for any $\F_{t-1}$ such that $\pi_{t,1} \geq 1/4$, we have $\E \left[X_t\mid \F_{t-1} \right] = \E \left[\indicator\left[I_t=2 \text{ or } \pi_{t,1} \geq \frac{1}{4} \right] \mid \F_{t-1}\right] =  1 >q$. Moreover, for any $\F_{t-1}$ such that $\pi_{t,1} < 1/4$, we have $\pi_{t,2} \geq 3/4$ and hence $\E\left[\indicator\left[I_t=2\right] \,\middle|\, \F_{t-1}\right] = \tilde{\pi}_{t,2} = (1-\gamma)\pi_{t,2} + \frac{\gamma}{2} \geq \frac{3-\gamma}{4}$ , therefore
    \begin{align*}
      \E \left[X_t\mid \F_{t-1} \right] =  \E \left[\indicator\left[I_t=2\right] \text{ or } \indicator[\pi_{t,1} \geq \frac{1}{4}] \,\middle|\, \F_{t-1}\right]=\E\left[\indicator\left[I_t=2\right] \,\middle|\, \F_{t-1}\right]\geq \frac{3-\gamma}{4}.
    \end{align*} 
    Note that we have 
    \begin{align*}
      \sum_{t=T_1+1}^{T_1+T_2} A_t = \sum_{t=T_1+1}^{T_1+T_2} \E \left[X_t\mid \F_{t-1} \right] \leq\sum_{t=T_1+1}^{T_1+T_2} q =  q T_2 =: \mu .
    \end{align*}
    We now apply Theorem~15 from \cite{kuszmaul2021multiplicative} to get for any $\varepsilon > 0$,
    \begin{align*}
      \Pr(Z_{T_1+T_2} - Z_{T_1} \leq -\varepsilon \mu \mid \F_{T_1}) \leq \exp \left( -\frac{\varepsilon^2 \mu}{2 c} \right)
    \end{align*}
    for $\F_{T_1}$ where $\mathcal{E}_1$ holds.
    
    Plugging in our setting of $c$ and using our definition of the martingale sequence gives
    \begin{align*}
      \Pr \left( \sum_{t=T_1+1}^{T_1+T_2} X_t \leq \sum_{t=T_1+1}^{T_1+T_2} \E \left[ X_t\,\middle|\, \F_{t-1} \right]  - \varepsilon q T_2  \,\middle|\, \F_{T_1}\right) 
      \leq \exp \left( -\frac{\mu  \varepsilon^2}{2} \right) .
    \end{align*}
    Using $\sum_{t=T_1+1}^{T_1+T_2} \E \left[ X_t\mid \F_{t-1} \right] \leq \mu$,
    we have all $\varepsilon >0$,
    \begin{align*}
      \Pr \left( \sum_{t=T_1+1}^{T_1+T_2} X_t \leq (1 - \varepsilon) \mu \,\middle|\, \F_{T_1} \right) 
      \leq \exp \left( -\frac{\mu \varepsilon^2}{2} \right) .
    \end{align*}
    This implies that, for any given $\F_{T_1}$ such that $\mathcal{E}_1$ holds, with probability at least $1-\frac{1}{T^2}$, we have 
    \begin{align*}
      \sum_{t=T_1+1}^{T_1 + T_2} X_t \geq  \mu = \frac{3-\gamma}{4}(1-\varepsilon_2)T_2,
    \end{align*}
    where $\varepsilon_2 = \sqrt{\frac{4 \ln{T}}{\frac{3-\gamma}{4}T_2}}$. 
    Now, recall $T'$ from Definition~\ref{def:def-h-prime}. For large enough $T$, by \eqref{eq:some-eq14}, we have $T_2\geq T'$. Therefore,  
    \begin{align*}
      (\frac{3-\gamma}{4} ) (1-\varepsilon_2) T_2 \geq (\frac{3-\gamma}{4})  (1-\varepsilon_2) T'.
    \end{align*}
    Also by definition of $T'$, we get
    \begin{align*}
      \frac{3-\gamma}{4} (1-\varepsilon_2) T'&= M \frac{2}{\eta}.
    \end{align*}  
    Finally by definition of $M'$, we have 
    \begin{align*}
      M \frac{2}{\eta} \geq M' \frac{2}{\eta},
    \end{align*}
    which means with probability at least $1-\frac{1}{T^2}$, \eqref{eq:some-eq27} happens.
  \end{enumerate}
\end{proof}
\subsubsection{Proof of Lemmas~\ref{lm:pi-2-at-the-end-expectation}~and~\ref{lm:prob-of-lowe-pi}}
We first prove Lemma~\ref{lm:pi-2-at-the-end-expectation}.
\begin{proof}[Proof of Lemma~\ref{lm:pi-2-at-the-end-expectation}]
  Consider any round $t$ where $t \geq t_0 + 1$. Consider any history $\mathcal{F}_{t-1}$ where event $\mathcal{E}_2$ happened. By applying  Lemma~\ref{lm:recursive-expectation-lemma} for $i=2$, we get
  \begin{align*}
    \E\left[\pi_{t+1,2} \mid \F_{t-1}\right]    &= (1-\eta) \,\pi_{t,2}  + \eta \, \pi^2_{t,2}.
  \end{align*}
  Now, note that since $\mathcal{E}_2$ happened we have $\pi_{t_0,1}\geq 1/4$. Since $\pi_{t,1} = \pi_{t_0+\tau}$ can only increase, we have $\pi_{t,1}\geq 1/4$. This implies $\pi_{t,2}\leq 3/4$, therefore
  \begin{align*}
    \pi^2_{t,2} \leq  \frac{3}{4} \, \pi_{t,2}.
  \end{align*}
  Therefore, we get
  \begin{align}
    \E\left[\pi_{t+1,2} \mid \F_{t-1}\right]    &\leq (1-\eta) \, \pi_{t,2}  + \frac{3}{4} \eta \, \pi_{t,2}  \nonumber\\
                                                &\leq (1-\frac{\eta}{4}) \, \pi_{t,2} . 
                                                  \label{eq:some-eq38}
  \end{align}
  We now can show an upper bound on $\E \left[\pi_{t+1,2} \,\middle|\, \mathcal{E}_2\right]$ by noting that 
  \begin{align*}
    \E \left[\pi_{t+1,2} \,\middle|\, \mathcal{E}_2\right] &=\E \left[ \E \left[\pi_{t+1,2}  \,\middle|\, \mathcal{E}_2, \F_{t-1}\right]\,\middle|\, \mathcal{E}_2\right] \\
                                                           &= \E \left[ \E \left[\pi_{t+1,2}  \,\middle|\, \F_{t-1}\right]\,\middle|\, \mathcal{E}_2\right] .
  \end{align*}
  This means that we can take the conditional expectation on both sides of \eqref{eq:some-eq38} to get
  \begin{align}
    \E[\pi_{t+1,2}|\mathcal{E}_2] \leq (1-\frac{\eta}{4}) \E \left[\pi_{t,2}|\mathcal{E}_2\right] .
    \label{eq:some-eq15}
  \end{align}
  Moreover, by definition of $\mathcal{E}_2$ we have $\E[\pi_{t_0, 2}|\mathcal{E}_2] = \pi_{t_0, 2} \leq 3/4$. Therefore, we get
  \begin{align*}
    \E[\pi_{t,2} \mid \mathcal{E}_2] &= \E[\pi_{t_0, 2}\mid\mathcal{E}_2] \prod_{s=t_0}^{t_0+\tau-1} \frac{\E[\pi_{s+1,2}\mid\mathcal{E}_2]}{\E[\pi_{s,2}\mid\mathcal{E}_2]}\\
                                     &\leq  \E[\pi_{t_0, 2} \mid\mathcal{E}_2] \prod_{s=t_0}^{t_0+\tau-1} (1 - \frac{\eta}{4}) &\text{by } \eqref{eq:some-eq15}  \\
                                     & \leq \frac{3}{4}  \, (1 - \frac{\eta}{4})^{\tau} \\
                                     &\leq \frac{3}{4} \exp{(-\frac{\eta}{4} \tau)} . & (1-x \leq e^{-x}) 
  \end{align*}
\end{proof}
We now prove Lemma~\ref{lm:prob-of-lowe-pi}.
\begin{proof}[Proof of Lemma~\ref{lm:prob-of-lowe-pi}]
  Let $\tau = T_3 +T_4 = c\,T $ for $c>0$. Define random variable $X := \pi_{T+1,1}\in [0,1]$. Clearly 
  Lemma~\ref{lm:pi-2-at-the-end-expectation} implies that $\E[X|\mathcal{E}_2]=\E[\pi_{T+1,1}|\mathcal{E}_2] = 1 - \E[\pi_{T+1,2}|\mathcal{E}_2]  \geq 1 - \frac{3}{4} e^{-\frac{\eta}{4}c\,T}$. Therefore, using Chebyshev's inequality, we get 
  \begin{align*}
    \Pr\left(X \leq \frac{3}{4}\,\middle|\,\mathcal{E}_2\right) &= \Pr\left(X - \E\left[X|\mathcal{E}_2\right] + \E\left[X|\mathcal{E}_2\right] \leq \frac{3}{4}\,\middle|\,\mathcal{E}_2\right) & \\
                                                                &= \Pr\left(X - \E\left[X|\mathcal{E}_2\right]   \leq \frac{3}{4} - \E\left[X|\mathcal{E}_2\right] \,\middle|\, \mathcal{E}_2\right)  \\
                                                                &\leq \Pr\left(X - \E\left[X|\mathcal{E}_2\right] \leq \frac{3}{4} - \left(1 - \frac{3}{4} e^{-\frac{\eta}{4}cT}\right)\,\middle|\,\mathcal{E}_2\right) & (\E\left[X|\mathcal{E}_2\right] \geq 1 - \frac{3}{4} e^{-\frac{\eta}{4}cT})\\
                                                                &= \Pr\left(X - \E\left[X|\mathcal{E}_2\right] \leq \frac{3}{4} e^{-\frac{\eta}{4}cT} - \frac{1}{4} \,\middle|\, \mathcal{E}_2\right) \\
                                                                &\leq \Pr\left(X - \E\left[X|\mathcal{E}_2\right]  \leq -\frac{1}{5}\,\middle|\, \mathcal{E}_2\right) & (\text{for large } T\text{ we have}\frac{3}{4} e^{-\frac{c}{4}\eta T} \leq \frac{1}{20})\\
                                                                &\leq \Pr\left( \left| X - \E\big[X|\mathcal{E}_2\big] \right| \geq \frac{1}{5} \,\middle|\, \mathcal{E}_2\right) &\\
                                                                &\leq 25 \operatorname{Var}(X|\mathcal{E}_2) & (\text{Chebyshev inequality})\\
                                                                &= 25 \left(\E\left[X^2|\mathcal{E}_2\right] - \E\left[X|\mathcal{E}_2\right]^2\right) &\\
                                                                &\leq  25  \left(\E\left[X|\mathcal{E}_2\right] - \E\left[X|\mathcal{E}_2\right]^2\right) & (E[X] \geq E[X^2] \text{ for X } \in [0,1] )\\
                                                                &= 25 \, \E\left[X|\mathcal{E}_2\right] \left(1-\E\left[X\,\middle|\,\mathcal{E}_2\right]\right)& \\
                                                                &\leq 25 \,\big(1-\E\big[X|\mathcal{E}_2\big]\big) & (\E\big[X|\mathcal{E}_2\big] \leq 1)&\\
                                                                &\leq \frac{75}{4} e^{-c  \frac{\eta}{4} T} . & (\text{from Lemma~\ref{lm:pi-2-at-the-end-expectation}})
  \end{align*}
\end{proof}
\subsubsection{Proof of Claim~\ref{claim:small-lawn-pi-on-expectation}}
Now, we are ready to prove Claim~\ref{claim:small-lawn-pi-on-expectation}. We recall Claim~\ref{claim:small-lawn-pi-on-expectation}. 
\RestateClaimOne*

First, we have the following simple observation.
\begin{observation}
  When running WSU-UX with any valid hyperparameter $\eta, \gamma$ on the loss sequence $\{\ell_t\}_{t=1}^T$ defined in Definition~\ref{df:loss-sequence}, we have with probability 1, that 
  \begin{align*}
    \pi_{T+1,1} \geq \left( \frac{1}{2}\right)^{\frac{T}{100}+1}.
  \end{align*}
  \label{obs:pi-1-loose-lower-bound}
\end{observation}
\begin{proof}[Proof of Observation~\ref{obs:pi-1-loose-lower-bound}]
  The probability of $\pi_{t,1}$ can only decrease in the first $\frac{T}{100}$ rounds and only if arm 1 is pulled in those rounds. It is easy to see that the value drops by at most a factor of 2 each time it is pulled as for $1 \leq s \leq T_1$ we have
  \begin{align*}
    \frac{\pi_{s+1,1}}{\pi_{s,1}} &= 1 - \eta \Big(\hat{\ell}_{s,1} - \sum_{j \in [2]}\pi_{s,j}\hat{\ell}_{s,j}\Big) \geq 1 -  \frac{\eta K}{\gamma}\indicator[I_s = 1]\nonumber \geq 1/2.
  \end{align*}
\end{proof}
\begin{proof}[Proof of Claim~\ref{claim:small-lawn-pi-on-expectation}]
  
  Define event $A := \indicator\left[\pi_{T+1,1}\geq 3/4\right]$. Using conditional expectation, we have
  \begin{align*}
    \E[\ln{\pi_{T+1,1}}\mid\mathcal{E}_2] &= \Pr(A\mid\mathcal{E}_2) \, \E[\ln{\pi_{T+1,1}}\mid \mathcal{E}_2,A] + \Pr(A^c \mid \mathcal{E}_2) \, \E[\ln{\pi_{T+1,1}}\mid\mathcal{E}_2,A^c] \\
                                          &\geq \left(1 - \frac{75}{4} e^{-\frac{c\eta}{4}T}\right) \, \E[\ln{\pi_{T+1,1}}|\mathcal{E}_2,A]  + \left(\frac{75}{4} e^{-\frac{c\eta}{4}T}\right) \,  \E[\ln{\pi_{T+1,1}}|\mathcal{E}_2,A^c] ,
  \end{align*}
  where the inequality comes from Lemma~\ref{lm:prob-of-lowe-pi}. This can be further lower bounded by 
  \begin{align}
    & \left(1 - \frac{75}{4} e^{-\frac{c\eta}{4}T}\right) \, \ln{\frac{3}{4}}  + \left(\frac{75}{4} e^{-\frac{c\eta}{4}T}\right) \,  \E[\ln{\pi_{T+1,1}}|\mathcal{E}_2,A^c] \nonumber\\
    &\geq \left(1 - \frac{75}{4} e^{-\frac{c\eta}{4}T}\right) \, \ln{\frac{3}{4}}  + \left(\frac{75}{4} e^{-\frac{c\eta}{4}T}\right) \, \min\left[\ln{\pi_{T+1,1}}\right] \nonumber\\
    &\geq \left(1 - \frac{75}{4} e^{-\frac{c\eta}{4}T}\right) \, \ln{\frac{3}{4}}  + \underbrace{\left(\frac{75}{4} e^{-\frac{c\eta}{4}T}\right)  \left(\frac{T}{100}+1\right) \ln{\frac{1}{2}}}_{\text{second term}} & (\text{Observation~\ref{obs:pi-1-loose-lower-bound}})\nonumber\\
    &\geq \ln{\frac{11}{16}} . & \text{second term }\rightarrow \text{0 as T} \rightarrow \infty \nonumber
  \end{align}

  Therefore, we have 
  \begin{align}
    \E[\ln{\pi_{T+1,1}}\mid\mathcal{E}_2]  \geq \ln{\frac{11}{16}}.  
    \label{inequality:inequality-for-ln-pi-T-given-H}
  \end{align}
  Now, we can lower bound $\E[\ln{\pi_{T+1,1}}]$ using conditional expectation:
  \begin{align*}
    \E\left[\ln{\pi_{T+1,1}}\right] &= \Pr(\mathcal{E}_2) \,  \E\left[\ln{\pi_{T+1,1}}\,\middle|\,\mathcal{E}_2\right] + \Pr(\mathcal{E}_2^c) \, \E\left[\ln{\pi_{T+1,1}}\,\middle|\,\mathcal{E}_2^c\right] \\
                                    & \geq \left(1-2\delta\right) \, \E\left[\ln{\pi_{T+1,1}}\,\middle|\,\mathcal{E}_2\right] + (2\delta) \, \E\left[\ln{\pi_{T+1,1}}\,\middle|\,\mathcal{E}_2^c\right] \\
                                    &\geq (1-2\delta) \, \E\left[\ln{\pi_{T+1,1}}\,\middle|\,\mathcal{E}_2\right] + (2\delta) \, \left[\min\ln{\pi_{T+1,1}}\right] \\
                                    &\geq (1-2\delta) \, (\ln{\frac{11}{16}}) + (2\delta) \, [\min\ln{\pi_{T+1,1}}] &(\text{By inequality~\ref{inequality:inequality-for-ln-pi-T-given-H}}) \\
                                    &  \geq (1-2\delta) \, (\ln{\frac{11}{16}}) + (2\delta) \, (\frac{T}{100}+1) \ln{\frac{1}{2}}  & (\text{Observation~\ref{obs:pi-1-loose-lower-bound}})\\
                                    &\geq (\ln{\frac{11}{16}}) + \underbrace{(2\delta) \, (\frac{T}{100}+1) \ln{\frac{1}{2}}}_{\text{second term}}  \\
                                    &\geq (\ln{\frac{5}{8}}) . & \delta = \frac{1}{T^2} \text{, therefore  second term} \rightarrow 0 \text{ as T}\rightarrow\infty 
  \end{align*}
  Therefore, we get
  \begin{align*}
    \E[\ln{\pi_{T+1,1}} + \ln{K}]  &\geq (\ln{\frac{5}{8}}) + \ln{K} \\
                                   &= \ln{\frac{5}{4}}. & (K=2)
  \end{align*}
  Therefore, for $c_1= \ln{\frac{5}{4}} > 0$, 
  \begin{align*}
    \E[\ln{\pi_{T+1,1}} + \ln{K}] \geq  c_1.
  \end{align*}
\end{proof}
\subsubsection{Proof of Claim~\ref{claim:pi-tilda-expectation-to-pi-expectation-conversion-lemma}}
We first prove Lemma~\ref{lm:dominant-stage} and then we prove  Claim~\ref{claim:pi-tilda-expectation-to-pi-expectation-conversion-lemma}.
\begin{proof}[Proof of Lemma~\ref{lm:dominant-stage}]
  Set $t = \frac{31T}{100} = T_1 + T_2 + T_3 +1 $. We can use conditional expectation for $\pi_{2,t}$ on event $\mathcal{E}_2$ defined in Definition~\ref{df:event-h-definition} to get
  \begin{align*}
    \E[\pi_{t,2}] &= \Pr(\mathcal{E}_2) \E[\pi_{t,2}|\mathcal{E}_2] + \Pr(\mathcal{E}_2^c) \E[\pi_{t,2}|\mathcal{E}_2^c] \\
                  &\leq \E[\pi_{t,2}|\mathcal{E}_2] + \Pr(\mathcal{E}_2^c) \cdot 1 . & (\pi_{t,2} \leq 1)
  \end{align*}
  Now, by setting $\tau = T_3$ in Lemma~\ref{lm:pi-2-at-the-end-expectation}, one would get $\E\left[\pi_{t,2}|\mathcal{E}_2\right] \leq \frac{3}{4}\exp{(-\frac{\eta}{4}T_3)} = \frac{3}{4}\exp{(\frac{-c \eta }{4}T)}  $ for $c = \frac{1}{10} > 0 $. Moreover, by Lemma~\ref{lm:high-prob-for-phase2.1} we have that $\Pr(\mathcal{E}_2^c)\leq \frac{2}{T^2}$. Therefore, for large enough $T$, we can further upper bound $\E[\pi_{t,2}]$ by 
  \begin{align*}
    \E\left[\pi_{t,2}\mid \mathcal{E}_2\right] + \Pr(\mathcal{E}_2^c) \cdot 1 \leq \frac{3}{4} e^{\frac{-c\eta T}{4}} + \frac{2}{T^2} \leq \frac{1}{4}.
  \end{align*}
\end{proof}
\begin{proof}[Proof of Claim~\ref{claim:pi-tilda-expectation-to-pi-expectation-conversion-lemma}]
  By expanding $\tilde{\pi}_{t,j}=\pi_{t,j}(1-\gamma) + \frac{\gamma}{2}$, we get
  \begin{align*}
    \left[\sum_{j=1}^2\tilde{\pi}_{t,j} \hat{\ell}_{t,j} - \hat{\ell}_{t,1} \right] = \sum_{j=1}^2\left(\pi_{t,j}(1-\gamma) + \frac{\gamma}{2}\right) \hat{\ell}_{t,j} - \hat{\ell}_{t,1} = \left[  \sum_{j=1}^2\pi_{t,j} \hat{\ell}_{t,j} - \hat{\ell}_{t,1}\right] - \sum_{j=1}^2\gamma \pi_{t,j}\hat{\ell}_{t,j} +  \sum_{j=1}^2\frac{\gamma}{2} \hat{\ell}_{t,j}
  \end{align*}
  Taking the expectation from both sides, we get
  \begin{align*}
    \E\left[\left(\sum_{j=1}^2\tilde{\pi}_{t,j} \hat{\ell}_{t,j} - \hat{\ell}_{t,1}\right)\right] 
    &= \mathbb{E}\left[\left(\sum_{j=1}^2\pi_{t,j} \hat{\ell}_{t,j} - \hat{\ell}_{t,1}\right)\right] + \left(- \sum_{j=1}^2 \E\left[ \gamma \pi_{t,j}\hat{\ell}_{t,j}\right] +  \sum_{j=1}^2\mathbb{E}\left[\frac{\gamma}{2} \hat{\ell}_{t,j}\right]\right)\\
    &= \mathbb{E}\left[\left(\sum_{j=1}^2\pi_{t,j} \hat{\ell}_{t,j} - \hat{\ell}_{t,1}\right)\right] - \sum_{j=1}^2\mathbb{E}\left[\gamma \pi_{t,j}\,\E_{t-1}\left[\hat{\ell}_{t,j}\right]\right] +  \sum_{j=1}^2\mathbb{E}\left[\frac{\gamma}{2}\,\E_{t-1}\left[ \hat{\ell}_{t,j}\right]\right]\\
    &= \mathbb{E}\left[\left(\sum_{j=1}^2\pi_{t,j} \hat{\ell}_{t,j} - \hat{\ell}_{t,1}\right)\right] +  \left(- \sum_{j=1}^2\mathbb{E}\left[\gamma \pi_{t,j}\ell_{t,j}\right] + \sum_{j=1}^2  \frac{\gamma}{2} \ell_{t,j}\right).
  \end{align*}
  Summing over $T$ rounds, we get
  \begin{align*}
    \E\left[\sum_{t=1}^T\left(\sum_{j=1}^2\tilde{\pi}_{t,j} \hat{\ell}_{t,j} - \hat{\ell}_{t,1}\right)\right] =  \mathbb{E}\left[\sum_{t=1}^{T}\left(\sum_{j=1}^2\pi_{t,j} \hat{\ell}_{t,j} - \hat{\ell}_{t,1}\right)\right] +  \underbrace{\left(- \sum_{t=1}^{T}\sum_{j=1}^2\mathbb{E}\left[\gamma \pi_{t,j}\ell_{t,j}\right] + \sum_{t=1}^{T}\sum_{j=1}^2  \frac{\gamma}{2} \ell_{t,j}\right)}_{\Delta} ,
  \end{align*}	
  where we define
  \begin{align}
    \Delta:= - \sum_{t=1}^{T}\sum_{j=1}^2\mathbb{E}\left[\gamma \pi_{t,j}\ell_{t,j}\right] + \sum_{t=1}^{T}\sum_{j=1}^2  \frac{\gamma}{2} \ell_{t,j}.
    \label{eq:some-eq39}
  \end{align}
  Note that to prove Claim~\ref{claim:pi-tilda-expectation-to-pi-expectation-conversion-lemma}, we need to show that for large enough $T$,  
  \begin{align}
    \Delta  \geq c_3\gamma \, T
    \label{eq:some-eq54}
  \end{align}
  holds true.
  
  Note that in loss sequence $\{\ell_t\}_{t=1}^T$ for all rounds $t$, we have $\ell_{t,1} + \ell_{t,2} =1$; therefore 
  \begin{align}
    \sum_{t=1}^{T}\sum_{j=1}^2  \frac{\gamma}{2} \ell_{t,j} = \frac{\gamma \, T}{2}.
    \label{eq:some-eq41}
  \end{align}
  Moreover, for large enough $T$ we have 
  \begin{equation}
    \begin{aligned}
      \sum_{t=1}^{T}\sum_{j=1}^2\mathbb{E}\left[\gamma \pi_{t,j}\ell_{t,j}\right] &= \underbrace{\sum_{t \in \mathcal{T}_1}\E\left[\gamma\pi_{t,1}  \right]}_{\text{phase 1}} + \underbrace{\sum_{t \in \mathcal{T}_2 \cup \mathcal{T}_3}\mathbb{E}\left[\gamma \pi_{t,2}\right]}_{\text{phase 2.1 and phase 2.2} } 
      +
      \underbrace{\sum_{t \in \mathcal{T}_4}\mathbb{E}\left[\gamma \pi_{t,2}\right]}_{\text{phase 2.3}}\\
      &\leq T_1 \frac{\gamma}{2}  +\sum_{t \in \mathcal{T}_2 \cup \mathcal{T}_3}\mathbb{E}\left[\gamma \pi_{t,2}\right]+\sum_{t \in \mathcal{T}_4}\mathbb{E}\left[\gamma \pi_{t,2}\right] &(\pi_{t,1}\leq 1/2 \in \mathcal{T}_1) \\
      &\leq T_1 \frac{\gamma}{2}  +\gamma \, (T_2+T_3) +\sum_{t \in \mathcal{T}_4}\mathbb{E}\left[\gamma \pi_{t,2}\right] &(\pi_{t,2} \leq 1)\\
      &\leq T_1 \frac{\gamma}{2}  + \gamma\left(T_2 + T_3\right)  +T_4\, \frac{\gamma}{4}, & (\forall t \in \mathcal{T}_4, \E[\pi_{t,2}]\leq 1/4 \text{ when }T \text{ is large})
    \end{aligned}
    \label{eq:some-eq40}
  \end{equation}
  where the first inequality comes from the fact that $\pi_{t,1}\leq 1/2$ when $1 \leq t \leq T_1$. The third inequality comes from the fact that by Lemma~\ref{lm:dominant-stage}, for large enough $T$, we have $\E[\pi_{t,2}]\leq\frac{1}{4}$ when $t \in \mathcal{T}_4$. 
  
  By using 
  \eqref{eq:some-eq41} and \eqref{eq:some-eq40}, we have 
  \begin{align*}
    \Delta = - \sum_{t=1}^{T}\sum_{j=1}^2\mathbb{E}\left[\gamma \pi_{t,j}\ell_{t,j}\right] + \sum_{t=1}^{T}\sum_{j=1}^2  \frac{\gamma}{2} \ell_{t,j} &\geq - \left( T_1 \frac{\gamma}{2}  + (T_2+T_3) \, \gamma  +T_4\, \frac{\gamma}{4} \right) + T \frac{\gamma }{2} \\
                                                                                                                                                     &= -T \,  (\frac{1}{100} \frac{\gamma}{2} + \frac{3}{10} \gamma +  \frac{69}{100} \frac{\gamma}{4}) +  T \frac{\gamma }{2} \\
                                                                                                                                                     &= \frac{9}{400} \gamma \,T = c_3 \gamma \,T , 
  \end{align*}
  for large enough $T$, i.e.,~\eqref{eq:some-eq54} holds. This proves the claim. 
\end{proof}

\section{Potential analysis}
\label{app:linear-approx}
\subsection{WSU as a linear approximation of Hedge update}
As mentioned in the main text, the WSU update can be viewed as a linear approximation to the Hedge update. In this section, we briefly show this approximation argument.

Observe that by applying the linear approximation $f(x) \approx f(x_0) + f'(x_0) (x-x_0)$ for $f(x) = \exp(-x)$ and for $x = \eta \ell_{t,i}$ and $x_0= \eta \bar{\ell}_t$, where $\bar{\ell}_t:=\sum_{j} \pi_{t,j} \ell_{t,j}$,  we get 
\begin{align}
  \exp{(-\eta \ell_{t,i})} \approx \exp{\left(-\eta \bar{\ell}_{t}\right)} \cdot \left(1-\eta\left(\ell_{t,i}-\bar{\ell}_t\right)\right).
  \label{eq:some-eq55}
\end{align}
Note that Hedge updates weights by the LHS of \eqref{eq:some-eq55}.  Now, if we instead update the weights by RHS of \eqref{eq:some-eq55}, we get 
\begin{align*}
  w_{t+1,i} = w_{t,i} \cdot \exp{\left(-\eta \bar{\ell}_{t}\right)} \cdot \left(1-\eta\left(\ell_{t,i}-\bar{\ell}_t\right)\right).
\end{align*}
By defining $\pi_{t,i} := \frac{w_{t,i}}{\sum_{j \in [K]}w_{t,j}}$, we get
\begin{align*}
  \pi_{t+1,i} = \frac{w_{t+1,i}}{\sum_{j}w_{t+1,j}} &= \frac{\exp{\left(-\eta \bar{\ell}_{t}\right)} \left[w_{t,i} \cdot \left(1-\eta\left(\ell_{t,i}-\bar{\ell}_t\right)\right)\right]}{\exp{\left(-\eta \bar{\ell}_{t}\right)} \left[\sum_{j\in [K]}w_{t,j}  \cdot \left(1-\eta\left(\ell_{t,j}-\bar{\ell}_t\right)\right)\right]} \\
                                                    &= \frac{w_{t,i} \cdot \left(1-\eta\left(\ell_{t,i}-\bar{\ell}_t\right)\right)}{\sum_{j\in [K]}w_{t,j}  \cdot \left(1-\eta\left(\ell_{t,j}-\bar{\ell}_t\right)\right)}\\
                                                    &=\frac{w_{t,i} \cdot \left(1-\eta\left(\ell_{t,i}-\bar{\ell}_t\right)\right)}{\sum_{j\in [K]}w_{t,j}  }\\
                                                    &= \pi_{t,i} \left(1-\eta\left(\ell_{t,i}-\bar{\ell}_t\right)\right).
\end{align*}
Note that this recovers the WSU update.\footnote{The idea of linear approximation of hedge was noted by \cite{kivinen1997exponentiated} for a slightly different setting.}
\subsection{Completed version of potential argument of Section~\ref{sec:potential-section}}
In this subsection, for the convenience of the reader, we give a comprehensive explanation of the derivation of \eqref{eq:some-eq53} and \eqref{eq:some-eq52}.

In the potential analysis of Hedge which can be found in \cite{hazan2016introduction}, for any $i \in [K]$ and $t \in [T]$, we define $\Phi^{\mathrm{HEDGE}}_{t,i}~:=~w_{t,i}$ with $w_{t,i}$ and $\pi_{t,i}$ as in Definition~\ref{def:hedge-mwu-definition}. Moreover, assume that $w_{1,i}=1$.\footnote{This is slightly different than Definition~\ref{def:hedge-mwu-definition} where $w_{1,i}=1/K$. We can view it as dividing all weights by the same constant. This does not impact the behaviour of Hedge at all.} We also define  define $\Phi^{\mathrm{HEDGE}}_{t}~:=~\sum_{j\in [K]}w_{t,j}$. By non-negativity of $w_{t,i}$, we have
\begin{align}
  \frac{1}{\eta} \ln{(\Phi^{\mathrm{HEDGE}}_{T+1,i})}
  \leq \frac{1}{\eta} \ln{(\Phi^{\mathrm{HEDGE}}_{T+1})} .
  \label{eq:some-eq59}
\end{align}
It is easy to see that for any $t \in [T]$ we can write
\begin{align*}
  \Phi^{\mathrm{HEDGE}}_{t+1} = \Phi^{\mathrm{HEDGE}}_{t} \left(\sum_{j \in [K]} \pi_{t,j} \exp{(-\eta \ell_{t,j})}\right).  
\end{align*}
Note that we have
\begin{align*}
  \sum_{j \in [K]} \pi_{t,j} \exp{(-\eta \ell_{t,j})} &\leq 1 -\eta  
                                                        \sum_{j \in [K]} \pi_{t,j} \ell_{t,j} + \eta^2 \sum_{j \in [K]}\pi_{t,j} \left(\ell_{t,j}\right)^2 & (\exp{(-x)} \leq 1 - x + x^2 \text{ for } x \geq 0) \\
                                                      &\leq \exp{\left(-\eta  
                                                        \sum_{j \in [K]} \pi_{t,j} \ell_{t,j} + \eta^2 \sum_{j \in [K]}\pi_{t,j} \left(\ell_{t,j}\right)^2\right)}. & (\exp{(x)} \leq 1+ x)
\end{align*}
Therefore, we have 
\begin{align*}
  \Phi^{\mathrm{HEDGE}}_{t+1} \leq \Phi^{\mathrm{HEDGE}}_{t} \exp{\left(-\eta  
  \sum_{j \in [K]} \pi_{t,j} \ell_{t,j} + \eta^2\sum_{t \in [T]} \sum_{j \in [K]}\pi_{t,j} \left(\ell_{t,j}\right)^2\right)}. 
\end{align*}
By applying \eqref{eq:some-eq57} recursively, we get
\begin{align}
  \Phi^{\mathrm{HEDGE}}_{T+1} &\leq \Phi^{\mathrm{HEDGE}}_{1} \exp{\left(-\eta  
                                \sum_{t \in [T]}\sum_{j \in [K]} \pi_{t,j} \ell_{t,j} + \eta^2 \sum_{t \in [T]} \sum_{j \in [K]}\pi_{t,j} \left(\ell_{t,j}\right)^2\right)} \nonumber\\
                              &= \exp{ \left(\ln{K}   -\eta  
                                \sum_{t \in [T]}\sum_{j \in [K]} \pi_{t,j} \ell_{t,j} + \eta^2 \sum_{t \in [T]} \sum_{j \in [K]}\pi_{t,j} \left(\ell_{t,j}\right)^2 \right)},
                                \label{eq:some-eq57}
\end{align}
since $\Phi^{\mathrm{HEDGE}}_{1}=\sum_{j \in [K]} \frac{1}{K} = K$.

On the other hand, we have 
\begin{align}
  \Phi^{\mathrm{HEDGE}}_{T+1,i} = \Phi^{\mathrm{HEDGE}}_{T,i}\exp{(-\eta \ell_{T,i})} =\Phi^{\mathrm{HEDGE}}_{1,i} \exp{(-\eta \sum_{t \in [T]}\ell_{t,i})} =  \exp{\left(-\eta \sum_{t \in [T]}\ell_{t,i}\right)}.
  \label{eq:some-eq58}
\end{align}
We can upper bound the RHS of \eqref{eq:some-eq59} by  \eqref{eq:some-eq57} and lower bound the LHS of \eqref{eq:some-eq59} by \eqref{eq:some-eq58} to get
\begin{align*}
  - \sum_{t \in [T]} \ell_{t,i} \leq \frac{1}{\eta} \ln{(\Phi^{\mathrm{HEDGE}}_{T+1,i})}
  \leq \frac{1}{\eta} \ln{(\Phi^{\mathrm{HEDGE}}_{T+1})} \leq  -\sum_{t \in [T]}\sum_{j \in [K]} \pi_{t,j} \ell_{t,j}+  \frac{\ln{K}}{\eta}+ \eta \sum_{t \in [T]} \left[\sum_{j \in [K]}\pi_{t,j} \left(\ell_{t,j}\right)^2\right].
\end{align*}
Note that the above is the full version of \eqref{eq:some-eq52}.
Rearranging, we get \footnote{Note that the exploration term is an inevitable error incurred by both Hedge and WSU when they move toward the optimal point in the simplex $\Delta_{K}$ by the learning rate $\eta$. We call it exploration term as the algorithm is trying to explore and find the optimal point in the domain of the simplex.}
\begin{align*}
  \sum_{t \in [T]}\sum_{j \in [K]} \pi_{t,j} \ell_{t,j} -\sum_{t \in [T]} \ell_{t,i} \leq  \underbrace{\frac{\ln{K}}{\eta}}_{\text{exploration term}} + \eta \sum_{t \in [T]} \underbrace{\left[\sum_{j}\pi_{t,j} \left(\ell_{t,j}\right)^2\right]}_{\text{Second order error}}.
\end{align*}

For WSU, the potential is defined as $\Phi^{\mathrm{WSU}}_{t,i}~:=~\pi_{t,i}$ and $\Phi^{\mathrm{WSU}}_{t}~:=~\sum_{j\in [K]}\pi_{t,i}=1$. By non-negativity of $\pi_{t,i}$ we have
\begin{align}
  \frac{1}{\eta} \ln{(\Phi^{\mathrm{WSU}}_{T+1,i})}
  \leq \frac{1}{\eta} \ln{(\Phi^{\mathrm{WSU}}_{T+1})} = 0.
  \label{eq:some-eq60}
\end{align}
Now, the RHS of \eqref{eq:some-eq51} (which is $0$) does not involve any second-order error term. In fact, since WSU is normalized, the RHS does not give us information about the regret.  However, we can extract the difference between the cumulative loss of the algorithm and expert $i$ from the LHS of \eqref{eq:some-eq51}.

Indeed, we have 
\begin{align}
  \Phi^{\mathrm{WSU}}_{T+1,i} &= \Phi^{\mathrm{WSU}}_{T,i} \left(1 - \eta\left[ \ell_{T,i} - \sum_{j} \pi_{T,j} \ell_{T,j}\right]\right)\nonumber\\
                              &= \Phi^{\mathrm{WSU}}_{1,i} \prod_{t \in [T]}\left(1 - \eta\left[ \ell_{t,i} - \sum_{j} \pi_{t,j} \ell_{t,j}\right]\right) \nonumber\\
                              &\geq \frac{1}{K} \prod_{t \in [T]} \exp{\left(-\eta\left[ \ell_{t,i} - \sum_{j} \pi_{t,j} \ell_{t,j}\right] - \eta^2 \left[ \ell_{t,i} - \sum_{j} \pi_{t,j} \ell_{t,j}\right]^2\right)}\nonumber\\
                              &= \frac{1}{K}  \exp{\left(-\eta\sum_{t \in [T]}\left[ \ell_{t,i} - \sum_{j} \pi_{t,j} \ell_{t,j}\right] - \eta^2 \sum_{t \in [T]}\left[ \ell_{t,i} - \sum_{j} \pi_{t,j} \ell_{t,j}\right]^2\right)} ,
                                \label{eq:some-eq61}
\end{align}
where the inequality comes from $1-x \geq \exp(-x - x^2)$ for $ 0 \leq x \leq 1/2$.

Using \eqref{eq:some-eq61}, we can lower bound the LHS of \eqref{eq:some-eq60} as
\begin{align*}
  \sum_{t \in [T]} \left[\sum_{j} \pi_{t,j} \ell_{t,j} - \ell_{t,i}\right]  
  -\frac{\ln{K}}{\eta} - \,\eta \sum_{t\in [T]} \left[\sum_{j} \pi_{t,j} \ell_{t,j} - \ell_{t,i}\right]^2 \leq \frac{1}{\eta} \ln{(\Phi^{\mathrm{WSU}}_{T+1,i})}
  \leq \frac{1}{\eta} \ln{(\Phi^{\mathrm{WSU}}_{T+1})} = 0.
\end{align*}
Note that the above is the full version of \eqref{eq:some-eq53}. Rearranging, we get
\begin{align*}
  \sum_{t \in [T]}\sum_{j \in [K]} \pi_{t,j} \ell_{t,j} -\sum_{t \in [T]} \ell_{t,i} \leq  \underbrace{\frac{\ln{K}}{\eta}}_{\text{exploration term}} + \eta \sum_{t \in [T]} \underbrace{\left[\sum_{j} \pi_{t,j} \ell_{t,j} - \ell_{t,i}\right]^2}_{\text{Second order error}}.
\end{align*}

\paragraph{Implication for bandit case}
In the bandit setting, when we use WSU-UX, we can show that we get a second-order term in \eqref{eq:some-eq53} which is upper bounded by
\begin{align*}
  \E\left[\left(\sum_{j} \pi_{t,j} \hat{\ell}_{t,j} - \hat{\ell}_{t,i}\right)^2\right]
  &\leq \E\left[\left(\sum_{j} \pi_{t,j} \hat{\ell}_{t,j}\right)^2+ \left(\hat{\ell}_{t,i}\right)^2\right]\\
  &\leq \E\left[\sum_{j} \pi_{t,j} \left(\hat{\ell}_{t,j}\right)^2+ \left(\hat{\ell}_{t,i}\right)^2\right]. & (\text{Jensen's inequality for } f(x) = x^2)
\end{align*}
Note that
\begin{align}
  \E \left[\sum_{j \in [K]} \pi_{t,j} \left(\hat{\ell}_{t,j}\right)^2 \right] &=
                                                                                \E \left[\sum_{j \in [K]} \pi_{t,j} \left(\frac{\ell_{t,j} \indicator[I_t=j]}{\tilde{\pi}_{t,j}}\right)^2 \right]  \nonumber\\
                                                                              &= \E \left[\sum_{j \in [K]} \pi_{t,j} \left(\frac{\ell_{t,j} }{\tilde{\pi}_{t,j}} \right)^2 \E_{t-1}\left[ \indicator[I_t = j]^2\right] \right] \nonumber\\
                                                                              &= \E \left[\sum_{j \in [K]}  \frac{\pi_{t,j} }{\tilde{\pi}_{t,j}}  \right] & (\ell_{t,j} \leq 1) \\
                                                                              &\leq 2K,
                                                                                \label{eq:some-eq56}
\end{align}
where the last inequality holds since $ \frac{\pi_{t,j} }{\tilde{\pi}_{t,j}} \leq 2$ as we have
\begin{align*}
  2 \tilde{\pi}_{t,i} - \pi_{t,i}  = 2 \left( (1-\gamma) \pi_{t,i} + \frac{\gamma}{K}\right) - \pi_{t,i} = (1-2\gamma) \, \pi_{t,i} +  2\gamma \frac{1}{K} \geq \min\{\pi_{t,i} , \frac{1}{K}\} \geq 0.
\end{align*}
Moreover, 
\begin{align*}
  \E \left[ \left(\hat{\ell}_{t,i}\right)^2 \right] &=
                                                      \E \left[ \left(\frac{\ell_{t,i} \indicator[I_t=i]}{\tilde{\pi}_{t,i}}\right)^2 \right]  \\
                                                    &= \E \left[\left(\frac{\ell_{t,i} }{\tilde{\pi}_{t,i}} \right)^2 \E_{t-1}\left[ \indicator[I_t = i]^2\right] \right] \\
                                                    &= \E \left[ \frac{\ell^2_{t,i}}{\tilde{\pi}_{t,i}}\right]\\
                                                    &\leq \E \left[ \frac{1}{\tilde{\pi}_{t,i}}\right]. & (\ell_{t,i}\leq 1)
\end{align*}
Therefore, we have 
\begin{align*}
  \E \left[\sum_{j \in [K]} \pi_{t,j} \left(\hat{\ell}_{t,j}\right)^2 \right]  \leq 2K + \E\left[\frac{1}{\tilde{\pi}_{t,i}}\right] \leq 2K + \frac{K}{\gamma} = O(\frac{K}{\gamma}),
\end{align*}
where the last inequality holds because we have $\tilde{\pi}_{t,i} = (1 - \gamma) \pi_{t,i} + \frac{\gamma}{K} \geq \frac{\gamma}{K}$.

\end{document}